\documentclass[pdflatex,sn-mathphys-num]{sn-jnl}%

\usepackage{graphicx}%
\usepackage{multirow}%
\usepackage{amsmath,amssymb,amsfonts}%
\usepackage{amsthm}%
\usepackage{mathrsfs}%
\usepackage[title]{appendix}%
\usepackage{xcolor}%
\usepackage{textcomp}%
\usepackage{manyfoot}%
\usepackage{booktabs}%
\usepackage{algorithm}%
\usepackage{algorithmicx}%
\usepackage{algpseudocode}%
\usepackage{listings}%
\usepackage{lineno}
\usepackage{float}
\usepackage{mathtools}
\usepackage{graphicx}
\usepackage{subcaption}

\theoremstyle{thmstyleone}%
\newtheorem{proposition}{Proposition}%

\definecolor{darkgreen}{rgb}{0.0,0.5,0.0}

\begin{document}

\title[Solving Inverse Problems of Chaotic Systems with Bidirectional Conditional Flow Matching]{Solving Inverse Problems of Chaotic Systems with Bidirectional Conditional Flow Matching}

\author[3]{\fnm{Peiyan} \sur{Hu}} \email{hupeiyan@westlake.edu.cn}
\intern{}

\author[7]{\fnm{Jian} \sur{Zhang}} \email{zhangjian191@mails.ucas.ac.cn}

\author[1]{\fnm{Jiashu} \sur{Pan}}\email{panjiashu@westlake.edu.cn}

\author[1]{\fnm{Ruiqi} \sur{Feng}}\email{fengruiqi@westlake.edu.cn}

\author[1]{\fnm{Tao} \sur{Zhang}}\email{zhangtao@westlake.edu.cn}

\author[3]{\fnm{Zhi-Ming} \sur{Ma}}\email{mazm@amt.ac.cn}

\author[4,5,6]{\fnm{Yuan-Sen} \sur{Ting}}\email{ting.74@osu.edu}

\author*[2]{\fnm{Gongjie} \sur{Li}}\email{gongjie.li@physics.gatech.edu}

\author*[1]{\fnm{Tailin} \sur{Wu}}\email{wutailin@westlake.edu.cn}

\affil[1]{\orgdiv{Department of Artificial Intelligence}, \orgdiv{School of Engineering}, \orgname{Westlake University}, \orgaddress{\city{Hangzhou}, \state{Zhejiang}, \postcode{310030},  \country{China}}}

\affil[2]{\orgdiv{Center for Relativistic Astrophysics, School of Physics}, \orgname{Georgia Institute of Technology}, \orgaddress{\city{Atlanta}, \postcode{GA 30313}, \country{USA}}}

\affil[3]{\orgdiv{Academy of Mathematics and Systems Science}, \orgname{Chinese Academy of Sciences}, \orgaddress{\city{Beijing}, \postcode{100190}, \country{China}}}

\affil[4]{\orgname{Max-Planck-Institut f\"ur Astronomie}, \orgaddress{\city{K\"onigstuhl}, \postcode{D-69117 Heidelberg}, \country{Germany}}}

\affil[5]{\orgdiv{Department of Astronomy}, \orgname{The Ohio State University}, \orgaddress{\city{Columbus}, \postcode{OH 43210}, \country{USA}}}

\affil[6]{\orgdiv{Center for Cosmology and AstroParticle Physics (CCAPP)}, \orgname{The Ohio State University}, \orgaddress{\city{Columbus}, \postcode{OH 43210}, \country{USA}}}

\affil[7]{\orgdiv{School of Astronomy and Space Sciences}, \orgname{Chinese Academy of Sciences}, \orgaddress{\city{Beijing}, \postcode{100049}, \country{China}}}

\abstract{

Modeling chaotic systems is crucial yet challenging. 
Inverse problems in chaotic dynamics, namely inferring initial conditions from final states, remain largely unsolved because of ill-posedness, non-uniqueness, instability, and potentially chaotic time-reverse dynamics.
We address this open problem with Bidirectional Conditional Flow Matching (Bi-CFM), 
which learns bidirectional mappings between distributions of initial and final states to capture the stochasticity of chaotic evolution and mitigate exponential error accumulation over time. 
Furthermore, for systems with conservation laws, we extend it to Conservation-constrained Bi-CFM (CBi-CFM).
Across the classic Lorenz, Circuit, and high-dimensional Lorenz 96 systems, Bi-CFM improves five distribution-level metrics over baselines while achieving a speedup of more than two orders of magnitude.
In the three-body planet-planet scattering problem\footnote{Planet-planet scattering is the process that planets gravitationally perturb each other, which leads to orbital variations and instability.}~ in planetary dynamics, CBi-CFM better respects conservation laws, with conservation errors comparable to those of the ground truth.
Finally, on real observations of globular clusters, collisional million-body systems shaped by $\sim 10^{10}$ years (10 Gyr) of evolution, our method represents an advance in accuracy, establishing a scalable route to solving inverse problems of long-timescale real-world chaotic dynamics.
}

\keywords{Chaotic system, Inverse problem, Generative model}

\maketitle

\section{Introduction}\label{sec.intro}
Chaotic systems are important yet complex nonlinear physical systems, arising in electrical systems \cite{hasler2005electrical, sprott2000simple}, climate systems \cite{tsonis1989chaos, broecker1995chaotic}, astronomy \cite{laskar1993chaotic, contopoulos2002order, barrow1977chaotic}, and biology \cite{olsen1977chaos, chialvo1990low}. Their intrinsic complexity poses fundamental challenges for modelling and analysis, primarily because small perturbations in the initial state can be amplified exponentially over time, as illustrated in Fig.~\ref{fig:fig1}\textbf{c}, where trajectories initialized with perturbations of order $10^{-6}$ rapidly diverge. 
In this work, we focus on the inverse problem of chaotic systems, namely, inferring initial conditions from target final observations. This long-standing challenge has broad implications in science \cite{Pecora1990SynchronizationIC, Kantas2013SequentialMC, Ouannas2016OnIP} and engineering \cite{Bennett1996GeneralizedIO, Gallet2021StructuralEF}, as recovering initial states is essential for understanding the mechanisms and histories of chaotic physical systems \cite{jung1999inverse}.
Despite its significance, this task is challenging, mainly due to three fundamental reasons. First, such inverse problems are ill-posed, meaning that their solutions may be non-unique or unstable \cite{kaipio2005statistical}. Second, one of the goals of inverse problems is to infer initial states whose evolved final states are close to target final states. However, errors in the initial states amplify during the forward evolution, leading to larger deviations from the target final states. Third, inverse inference may suffer from information loss when certain bodies or components of the system vanish during evolution. For example, a planet may be scattered out of the planetary system \citep{Rasio96,Goldreich04,Chatterjee08}, which causes the loss of its information and makes the reconstruction rely solely on the remaining components.

For a deterministic and non-chaotic physical system, once the governing equations of motion are known, traditional numerical methods can infer the initial conditions with small numerical errors. The convergence properties of these methods ensure that by starting from the target final state and integrating the system backward in time, one can obtain a high-fidelity estimate of the initial state whose forward evolution closely matches the target \cite{leveque1998finite, lu2024trace, binev2004adaptive}. However, in chaotic systems, randomness arises because deterministic nonlinear dynamics amplify infinitesimal perturbations in the initial conditions into unpredictable behaviors. Consequently, traditional numerical approaches fail to accurately recover the initial state even when the equations are known and integrated backward. In particular, when information loss occurs during forward evolution, for example, when certain bodies or components vanish from the system, traditional numerical methods may become entirely inapplicable. Recent progress in deep learning has opened new possibilities for addressing this challenge. Neural networks have shown the capability to model chaotic dynamics with acceleration, accuracy, and efficiency, leading to a growing number of approaches in this area. However, existing studies mainly focus on other tasks of chaotic systems, such as the prediction task \cite{du2024conditional, huwavelet, alibert2025transformer, wang2024interpretable}. Several studies have also focused on using deep learning to solve inverse problems, but they have not focused on the more challenging setting of chaotic systems \cite{raissi2019physics,yang2021b,bastek2023inverse,molinaro2023neural,huang2024diffusionpde}.

\begin{figure}[t]
    \centering
    \vspace{-15pt}
    \includegraphics[width=0.75\linewidth, trim=5 320 391 14, clip]{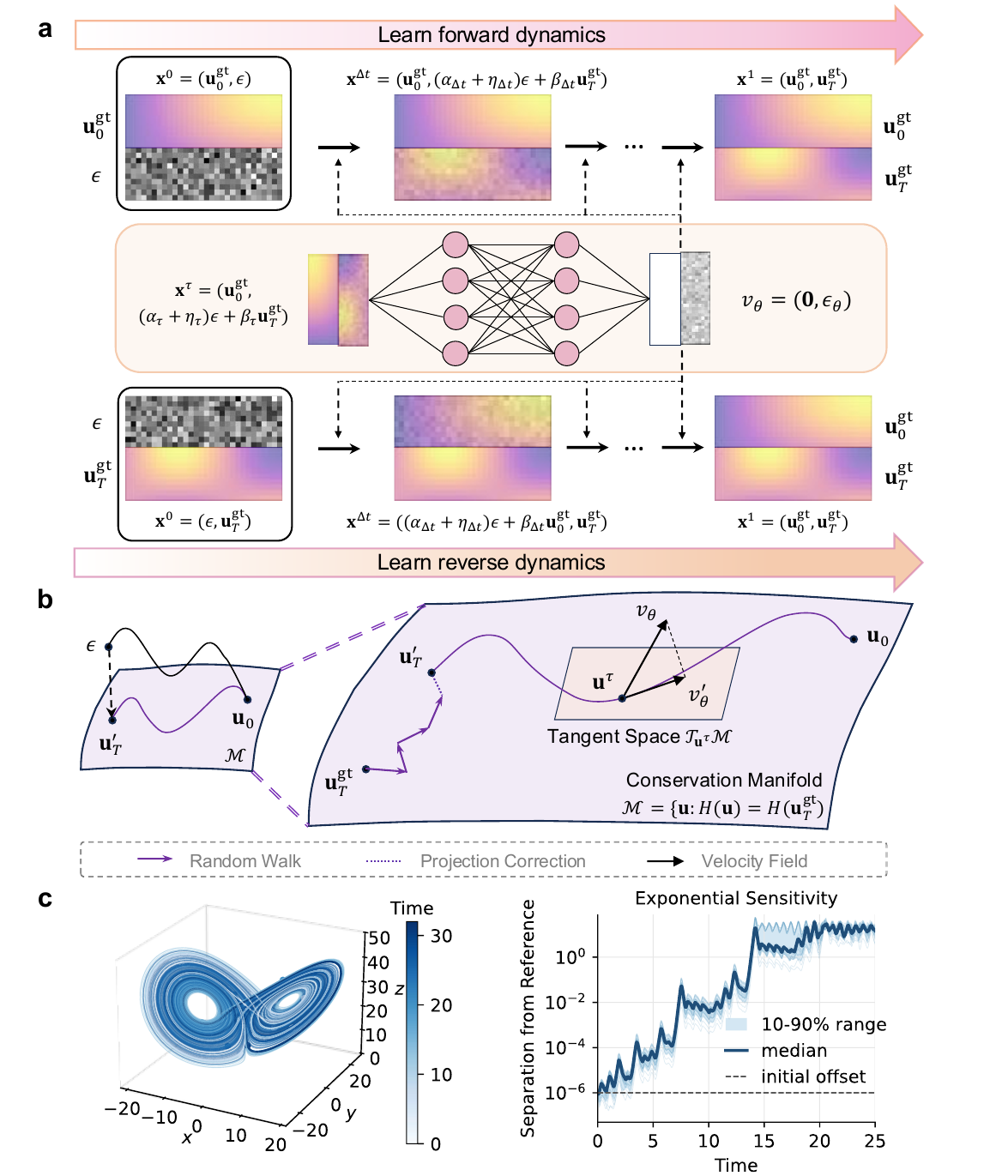}
    \caption{\textbf{Schematic illustration of Bi-CFM and CBi-CFM, and the sensitive dependence on initial conditions.} 
    \textbf{a. Bi-CFM framework.} Bi-CFM extends Conditional Flow Matching by jointly learning the forward and inverse conditional distribution $p(\mathbf{u}_T^\text{gt}|\mathbf{u}_0^\text{gt})$ (upper panel) and $p(\mathbf{u}_0^\text{gt}|\mathbf{u}_T^\text{gt})$ (lower panel). The noise-free conditioning state is concatenated with the current flow state as model input. 
    \textbf{b. CBi-CFM framework.} CBi-CFM constrains the probability flow within the conservation manifold $\mathcal{M}$ to enforce conservation quantity $H$. The prior distribution is sampled via random walks on $\mathcal{M}$, after which the projection correction is applied to mitigate numerical errors. In addition, the velocity is projected onto the tangent space $\mathcal{T}_{\mathbf{u}_t'}\mathcal{M}$, ensuring conservation throughout the flow.
    \textbf{c. Exponential sensitivity.} In the Lorenz system, perturbations of order $10^{-6}$ are rapidly amplified, causing initially nearby trajectories to diverge across the chaotic attractor.
    }
    \label{fig:fig1}
\end{figure}

In this work, as shown in Fig. \ref{fig:fig1}, we propose Bidirectional Conditional Flow Matching (Bi-CFM), an algorithm for the inverse problem of chaotic systems, and its extended version, Conservation-constrained Bi-CFM (CBi-CFM), which enforces conservation laws when conservation quantities are known in the system. We analyze and address the inverse problem of chaotic systems systematically. The method builds upon the probabilistic generative model Conditional Flow Matching (CFM) \cite{lipmanflow, liuflow}, the advantages of which stem from two key insights. First, under limited computational resources, a chaotic system can be regarded as a stochastic system \cite{berliner1992statistics, LEITH1996481, stone2019statistical}, and we therefore learn the distribution of its states rather than the deterministic trajectory. This probabilistic modeling framework is also more suitable for addressing the ill-posed problem, since the solutions may not be unique. 
Second, we directly learn the end-to-end mapping between initial and final states, rather than relying on time-dependent iterative updates as in traditional numerical solvers. This design avoids repeatedly accumulating numerical errors along the temporal dimension, making the inference error more controllable even as the horizon length increases. 
Without iterations proportional to the horizon length, this design mitigates the error accumulation in chaotic dynamics. In addition, we introduce a bidirectional modeling framework that jointly learns both forward and reverse dynamics, which allows the model to capture the global structure of chaotic systems rather than focusing solely on reproducing initial states. By enforcing mutual consistency between the forward evolution and the time-reverse dynamics, our approach mitigates the amplification of errors during forward evolution that typically arise in chaotic systems. Furthermore, for physical systems governed by conservation laws, we extend this framework to CBi-CFM, which enforces the learned probability flow path to evolve within the conservation manifold defined by the conservation principles. This extension enables the model to generate physically consistent solutions under chaotic conditions. 

We demonstrate the effectiveness of our method through experiments across diverse chaotic systems, including the classic and real-world ones. On the classic Lorenz system, Circuit system, and high-dimensional Lorenz 96 system, our proposed method outperforms other baselines under five distribution-level metrics, including the reverse integration based on traditional numerical solvers. Additionally, in the challenging three-body planetary system with scattering and collision, the inverse problem spans up to $10^{6}$ orbital periods and contains information loss. Our method successfully infers the distribution and causal relationships close to the ground truth. Moreover, CBi-CFM preserves energy conservation at the same order as the original data. To test our method in a real-world setting, we further study the inverse problem of long-term globular-cluster evolution from observational data. As ancient dense stellar systems whose initial states have been reshaped over $\sim 10$ Gyr of dynamical evolution, globular clusters encode the assembly history of the Milky Way and provide a challenging testbed for Galactic archaeology. Compared with a state-of-the-art Monte Carlo method, Bi-CFM yields final simulated profiles that better match the observations, both visually and quantitatively. These results show that Bi-CFM can serve as a scalable tool for solving inverse problems in long-timescale chaotic systems.

\section{Results}\label{sec.result}

The overview of our proposed model is shown in Fig. \ref{fig:fig1}. The Bi-CFM framework in Fig. \ref{fig:fig1}\textbf{a} is based on CFM, which learns the velocity field $v_\theta$ that defines the flow path from Gaussian noise to the data distribution. 
Building upon this, we introduce a bidirectional modeling strategy that simultaneously learns the forward and reverse dynamics. 
As for the Conservation-constrained Bi-CFM (CBi-CFM) framework in Fig. \ref{fig:fig1}\textbf{b}, we constrain the entire flow path within the conservation manifold $\mathcal{M}$ defined by the conservation law, ensuring that the inferred initial states satisfy the conserved quantity $H$. 
Specifically, we first obtain a prior distribution that satisfies the conservation law. Then, we constrain the velocity field to lie within the tangent plane $\mathcal{T}_{\mathbf{u}^\tau} \mathcal{M}$ of the conservation manifold $\mathcal{M}$, ensuring that the entire probability flow remains within $\mathcal{M}$.
More details are given in the Method Section (Section \ref{sec.meth}).

For clarity, we summarize here the notation used throughout this paper. The ground-truth trajectory is denoted as $\mathbf{u}_t^{\text{gt}}$ for $t = 0, \dots, T$, and the initial state inferred by the algorithm is represented as $\mathbf{u}_0$.
The trajectory obtained by forward evolving $\mathbf{u}_0$ using the same high-precision numerical solver is written as $g(\mathbf{u}_0, t)$, where $g(\cdot, t)$ denotes the mapping from the initial state to the state at time $t$. In addition, the evaluation metrics for solving inverse problems consist of three main aspects. The first measures how close the inferred initial state $\mathbf{u}_0$ is to the ground-truth initial state $\mathbf{u}_0^{\text{gt}}$. The second evaluates whether the final state obtained by forward evolving $\mathbf{u}_0$, denoted as $g(\mathbf{u}_0, T)$, remains close to the target final state $\mathbf{u}_T^{\text{gt}}$. The third is about the causal relationship between the states at different time steps.

\subsection{Evaluation on Three Classic Chaotic Systems}\label{sec.bicfm}

We first evaluate Bi-CFM on the Lorenz \citep{lorenz2017deterministic}, Circuit \citep{dadras2009novel}, and Lorenz 96 \citep{lorenz1996predictability} systems with increasing complexity, covering low-dimensional, multi-lobe, and high-dimensional chaotic dynamics.
CBi-CFM is not applied because these systems do not impose conservation constraints. The experiment setups and baselines are introduced in Section \ref{sec:exp_setting} and \ref{sec:baseline}. 

\textbf{Evaluation metrics.} Due to the randomness and unpredictability caused by chaos, deterministic metrics become meaningless as the evolution time increases \cite{berliner1992statistics, wang2024interpretable}. Therefore, we utilize distribution-level metrics, specifically the Wasserstein-2 distance (W-2 distance) and Kullback-Leibler divergence (KL divergence), to compare distributions. For the W-2 distance, we compare several sets of distribution distances, including between $\mathbf{u}_0$ and $\mathbf{u}_0^\text{gt}$, between $\mathbf{u}_T$ and $\mathbf{u}_T^\text{gt}$, between the entire trajectory $(\mathbf{u}_0, g(\mathbf{u}_0, 1), \dots, g(\mathbf{u}_0, T))$ and $(\mathbf{u}_0^\text{gt}, \mathbf{u}_1^\text{gt}, \dots, \mathbf{u}_T^\text{gt})$, and between the pairs of initial and target final states $(\mathbf{u}_0, \mathbf{u}_T^\text{gt})$ and $(\mathbf{u}_0^\text{gt}, \mathbf{u}_T^\text{gt})$. 
Here, we highlight that it is also important to assess whether the causal relationship between states at different time points conforms to the system's dynamics, as reflected in the last two metrics with states at different $T$. Furthermore, we consider the KL divergence between the pairs of initial and target final states, $(\mathbf{u}_0, \mathbf{u}_T^\text{gt})$ and $(\mathbf{u}_0^\text{gt}, \mathbf{u}_T^\text{gt})$, to further measure how well the inferred initial state aligns with the target final state in terms of the system's dynamic correlations. Details are given in Section \ref{sec.metric}.

\textbf{Results.} 
We evaluate different methods on the Lorenz, Circuit, and Lorenz 96 systems under several horizon lengths $T$ (measured in units of the number of periods), to examine how their ability to solve the inverse problem changes as the prediction horizon increases. The quantitative results are reorganized in Fig.~\ref{fig:2}. Fig.~\ref{fig:2}\textbf{a} shows the evolution of representative metrics across different horizon lengths. These include the W-2 distance of the inferred initial state distribution, the W-2 distance of the forward-evolved final state distribution, and the KL divergence of the paired initial and target final states. 
They respectively assess the accuracy of the inferred initial conditions, the accuracy of the final states obtained by forward evolving these initial conditions, and whether the causal relationship between the initial and final states is consistent with the underlying physical dynamics.
The remaining distribution-level metrics, including the W-2 distances of trajectories and initial--final state pairs when not shown in the main text, are provided in Supplementary Note~\textcolor{blue}{C}. These additional W-2 metrics show trends similar to the KL divergence. 
To provide a compact and explicit comparison across metrics, Fig.~\ref{fig:2}\textbf{b} shows radar plots at the longest horizon for each system. The values of each metric are normalized by the same metric-specific constant. Backward Integration is not shown in these radar plots because its values lie outside the displayed range.

\begin{figure}[t]
    \centering
    \includegraphics[width=\linewidth, trim=50 210 410 25, clip]{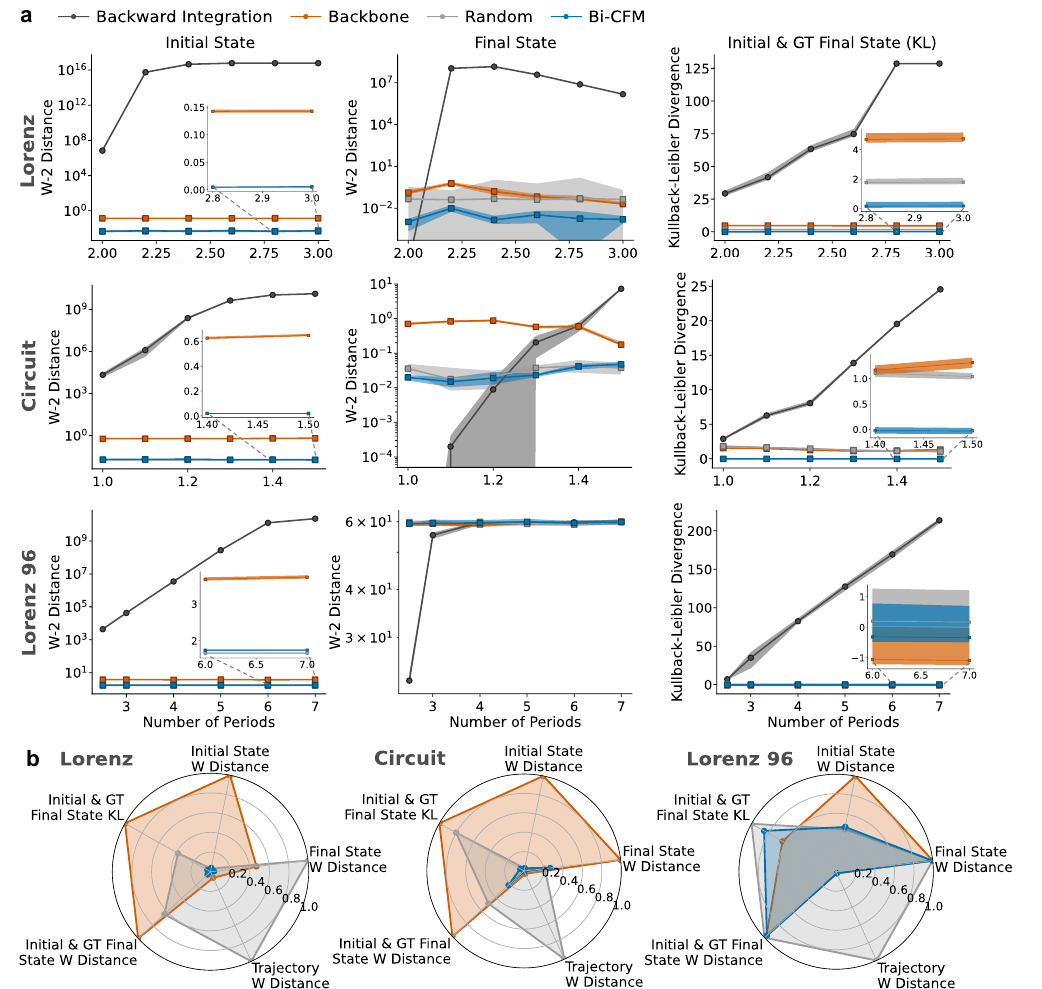} 
    \caption{\textbf{Evaluation metrics on the Lorenz, Circuit, and Lorenz 96 systems.} 
    \textbf{a.} Evaluation metrics of different methods with respect to the horizon length $T$. Shaded bands denote the 5–95\% range across repeated subsampling.
    \textbf{b.} Normalized results of evaluation metrics at the longest horizon. Results of Backward Integration lie entirely outside the circle.
    }
    \label{fig:2}
\end{figure}

As for the Lorenz and Circuit system, on the W-2 distances of initial and final state distributions, Bi-CFM performs close to Random, indicating that it successfully matches the target marginal distributions, since Random is constructed to match these marginals. However, Random ignores the causal relationship between the initial and final states. As a result, it performs poorly on metrics that depend on the joint or trajectory-level distributions, including the distributions of entire trajectories and initial--final state pairs. In contrast, Bi-CFM achieves better performance on these causal metrics. Besides, the superior performance of Bi-CFM over Backbone highlights the advantage of probabilistic generative modelling and bidirectional learning.

The dependence on horizon length further reveals the instability of iterative reverse-time numerical methods. Backward Integration can perform reasonably well when the horizon is short, where the system behaves more deterministically. However, as $T$ increases and chaoticity becomes more pronounced, its performance deteriorates rapidly. 
These results suggest that error accumulation in chaotic systems makes methods that repeatedly integrate over the time dimension less effective than end-to-end learned mappings, particularly at long horizons. By contrast, both Backbone and Bi-CFM remain comparatively stable as $T$ increases.

For the Lorenz 96 system, Bi-CFM also achieves the best overall performance, with the most notable advantage in matching the inferred initial state distribution. This is consistent with the training objective of CFM, which directly learns the target distribution of the inferred state. However, Backward Integration performs better than Random and other methods in the W-2 distance of the final state at smaller $T$. Since Random represents an optimal marginal baseline for this metric, this indicates that they all capture the final-state marginal distribution reasonably well. We also find that Backbone, Random, and Bi-CFM achieve comparable errors on the initial--final state-pair metrics, which is consistent with the low-dimensional visualizations discussed below. Nevertheless, Bi-CFM still outperforms the baselines in the remaining aspects.

\begin{figure}[t]
    \centering
    \includegraphics[width=\linewidth, trim=0 0 9 0, clip]{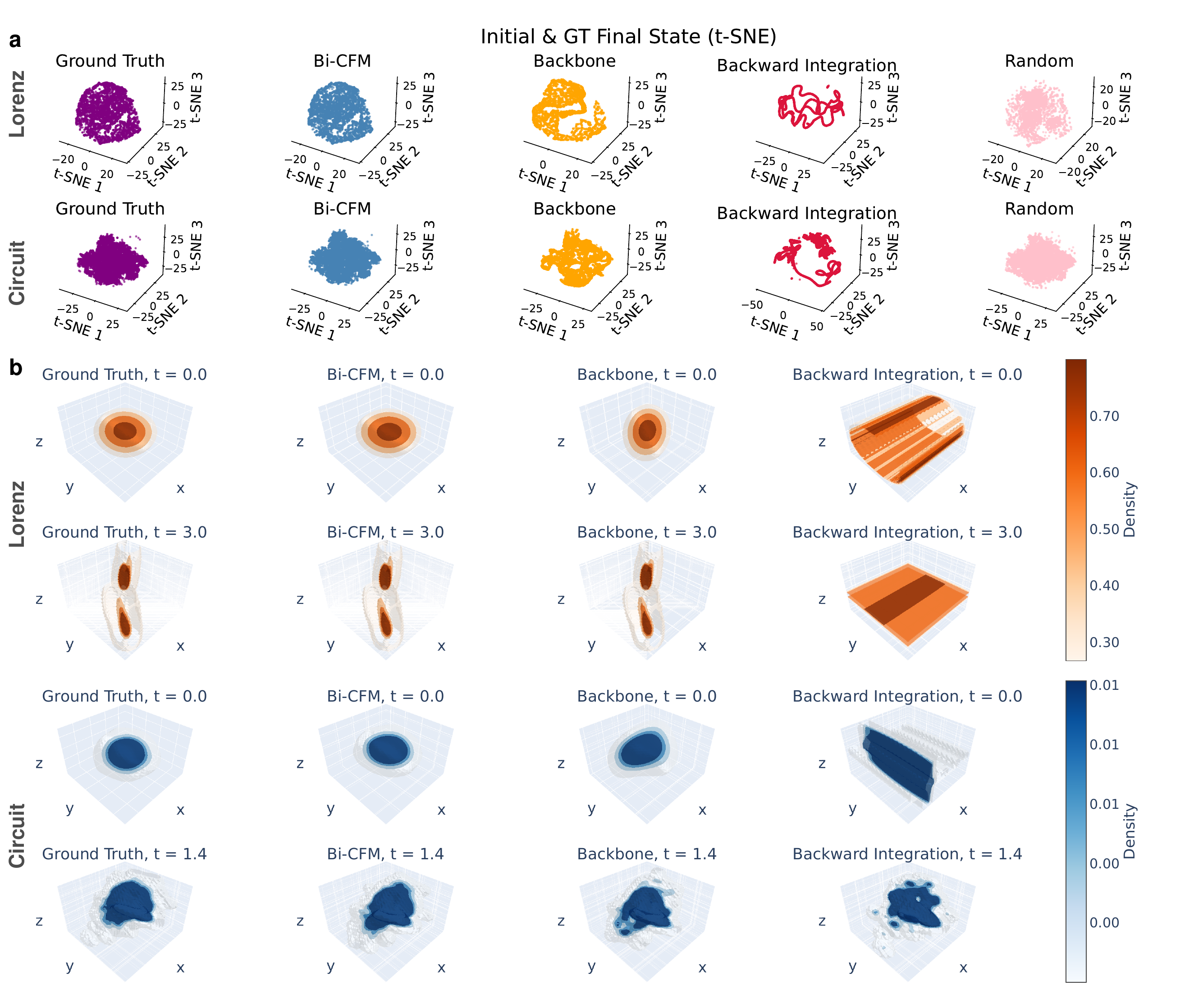} 
    \caption{\textbf{Visualizations of distributions on the Lorenz and Circuit systems.} 
    \textbf{a.} 3D t-SNE \cite{Maaten2008VisualizingDU} visualizations of paired inferred initial and target final states $(\mathbf{u}_0, \mathbf{u}_T^\text{gt})$ at a middle time step.
    \textbf{b.} Density isosurfaces of inferred initial states $\mathbf{u}_0$ and evolved final states $g(\mathbf{u}_0,T)$.
    }
    \label{fig:3}
\end{figure}

\begin{figure}[t]
    \centering
    \includegraphics[width=\linewidth, trim=7 0 11 0, clip]{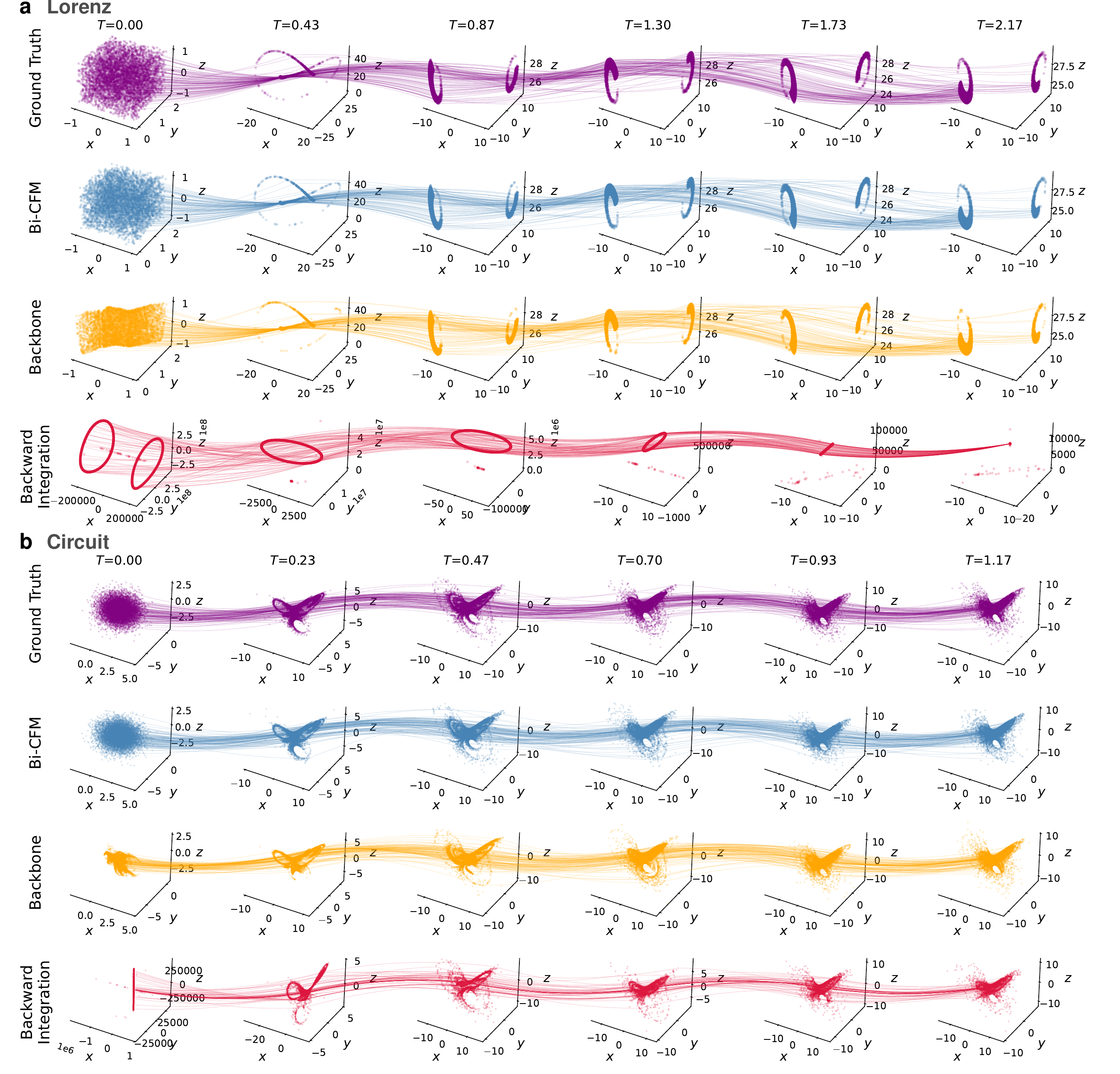} 
    \caption{\textbf{Evolved trajectories $g(\mathbf{u}_0, t)$ on the test set, starting from inferred initial states $\mathbf{u}_0$.} 
    \textbf{a.} Trajectories of the Lorenz system.
    \textbf{b.} Trajectories of the Circuit system.
    }
    \label{fig:4}
\end{figure}

We next visualize the distributions in Fig.~\ref{fig:3}. Fig.~\ref{fig:3}\textbf{a} shows the three-dimensional t-SNE embeddings of paired inferred initial states and target final states. The dimensionality reduction is performed using t-SNE \cite{Maaten2008VisualizingDU}, which provides an intuitive comparison between the joint causal distributions produced by different algorithms and the ground-truth distribution. For both systems, Bi-CFM closely matches the ground-truth distribution, whereas other methods exhibit visible deviations. In the Circuit system, Backward Integration shows the largest discrepancy. The magnitude of these discrepancies is consistent with the quantitative distributional metrics reported in Fig.~\ref{fig:2}. For Lorenz 96, owing to the high dimensionality and chaoticity of the system \citep{karimi2010extensive}, both the selected three-dimensional projections of the final states and the t-SNE embeddings of the initial--final state pairs do not show clearly discernible differences among methods. These visualizations reflect the intrinsic limitation of low-dimensional projections in revealing high-dimensional causal differences \citep{van2008visualizing,wattenberg2016how}.

Fig.~\ref{fig:3}\textbf{b} further compares the three-dimensional density isosurfaces of the inferred initial states and corresponding final states. For both the initial and final states, Bi-CFM produces distributions that are closest to the ground truth. Backbone produces distributions deviated from the ground truth, reflecting the limitation of a deterministic mapping. The corresponding two-dimensional $(x,y)$ distributions are provided in Supplementary Note~\textcolor{blue}{C}, where Backbone further shows a collapse towards a lower-dimensional one-dimensional manifold, essentially a line segment. Backward Integration is unstable, with inferred initial states showing the largest deviation from the target distribution.

Finally, we visualize the evolved trajectories in Fig.~\ref{fig:4}. Fig.~\ref{fig:4}\textbf{a} shows the Lorenz system at $T=2.6$, which is chosen as an intermediate horizon between $T=2.0$ and $T=3.0$. Fig.~\ref{fig:4}\textbf{b} shows the Circuit system at $T=1.4$. In both cases, we visualize 5000 trajectory points from the ground truth and from the forward-evolved trajectories $g(\mathbf{u}_0,t)$ of different algorithms. Bi-CFM produces trajectory distributions whose shapes are almost identical to the ground truth, whereas Backbone produces narrower distributions. Backward Integration yields unrealistic distributions with a much larger range, reaching the $10^5$ scale, and both its initial and final states deviate substantially from the target distributions. Nevertheless, in the Circuit system, despite its unrealistic initial states, its final distribution partially resembles the three-lobe attractor, highlighting the ill-posed nature of the inverse problem: substantially different initial states can lead to similar final-state distributions. For the Lorenz 96 system, trajectory visualizations are provided in Supplementary Note~\textcolor{blue}{C}.

\subsection{Evaluation on Three-body Planetary Systems}\label{sec:planet}

In this section, we apply the Bi-CFM and CBi-CFM methods to an astrophysical problem involving planet-planet dynamical interactions, including both scattering and collision events, as illustrated in Fig. \ref{fig:planet}\textbf{a}. This process is a natural consequence of planetary formation within protoplanetary disks, where nascent planets migrate inward and interact gravitationally, ultimately driving dynamical instabilities, close encounters, scattering, and collisions as the disks disperse \citep{Kokubo02, Goldreich04, Ida04}. Previous studies have shown that planet-planet scattering and collisions can explain the origin of a wide range of planetary architectures, including misaligned hot Jupiters in close proximity to their host stars, eccentric wide-orbit planets, and systems shaped by giant impacts \citep{Rasio96, Chambers96, Lin97, Adams03, Chatterjee08, Scharf09, Boss06, Mustill17}.
However, the initial dynamical states of planetary systems prior to instability remain poorly constrained. Determining these initial conditions will shed light on the origin of planets, which gives rise to a wide diversity of planetary systems, as well as a better understanding of the uniqueness of the solar system in the universe.

\textbf{Experiment setups.} For simulations, we considered systems consisting of three planets orbiting a solar-mass host star following the previous work \cite{Chatterjee08}. The scattering simulations are performed with \texttt{REBOUND} and \texttt{REBOUNDx}, incorporating general-relativistic corrections, mass- and momentum-conserving collisions, and planetary ejection. Further details of the system are presented in Section \ref{sec:exp_setting} and Supplementary Note \textcolor{blue}{A}.

Analyzing and understanding this system is meaningful, but solving its inverse problem is challenging because of the information loss phenomenon: once a planet is ejected or collides, part of the system becomes unobservable, and the initial conditions must be inferred solely from the remaining observable states. Although certain traditional numerical algorithms are time-reversible in principle \cite{lu2024trace}, reverse integration cannot be applied here due to the partial unobservability of the system. Consequently, the Backward Integration baseline is not applicable in this experiment.

In this system, energy conservation, \textit{i.e.}, that the total energy of the initial and final states is the same, serves as a fundamental physical constraint, and satisfying this conservation law is essential for accurately solving both the forward and inverse problems. We note that structure-preserving solvers, specifically symplectic integrators, are widely used in planetary dynamics \cite{rein2015whfast, mikkola1999explicit}. To account for this property, we evaluate the performance of CBi-CFMs, which explicitly enforces the conservation of energy in both the training and sampling process.

\textbf{Evaluation metrics.} As for the evaluation metrics, due to the requirement of enforcing conservation constraints, we introduce a new metric, termed the Relative Conservation Error (RCE). It quantifies how well the conservation law is satisfied and is defined as
\begin{equation*}
    \text{RCE} = \frac{|H - H_{\text{gt}}|}{|H_{\text{gt}}|},
\end{equation*}
where $H$ denotes the conserved quantity (\textit{e.g.}, total energy) computed from the inferred state, and $H_{\text{gt}}$ is the corresponding ground truth value. 

\begin{figure}[t]
    \centering
    \includegraphics[width=1.02\linewidth, trim=17 10 6 0, clip]{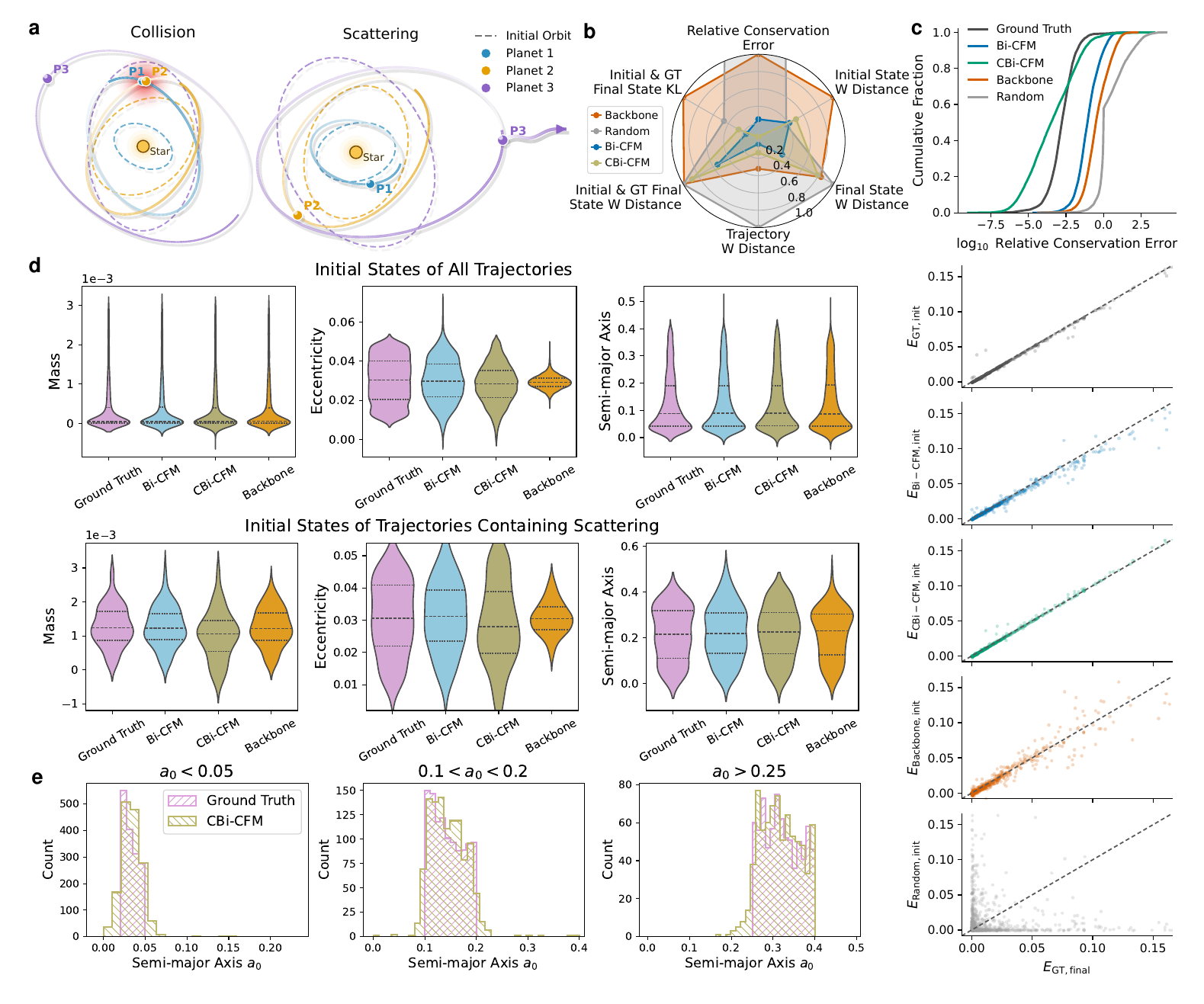}  
    \caption{\textbf{Visualizations, evaluation metrics, distributions, and energy conservations on the star-planet system.} 
    \textbf{a.} Visualizations of the collision and scattering.
    \textbf{b.} As Fig. \ref{fig:2}\textbf{b} but for the star-planet system with another metric, Relative Conservation Error.
    \textbf{c.} Cumulative distributions of Relative Conservation Error, and calibration scatter plots comparing inferred initial energy against ground-truth final energy.
    \textbf{d.} Distributions of the inferred initial features, mass, semi-major axis, and eccentricity of all trajectories and trajectories containing scattering.
    \textbf{e.} Distributions of the inferred initial semi-major axis $a_0$ on trajectories with small ($a_0<0.05$), medium ($0.1<a_0<0.2$), and large ($a_0>0.25$) initial ground truth semi-major axis $a_0^\text{gt}$.}
    \label{fig:planet}
\end{figure}

\textbf{Results.} Our experimental results are organized into four parts that correspond to the subpanels in Fig. \ref{fig:planet}. In the subpanel Fig.~\ref{fig:planet}\textbf{b}, we provide quantitative results for all methods across all metrics. The range of RCE of Random is beyond the normalized circle, and hence is omitted. First, we observe that Bi-CFM and CBi-CFM achieve the best performance, except for the Initial State W-2 distance, where Random matches the target initial distribution and thus is slightly better than Bi-CFM. Second, comparing CBi-CFM with Bi-CFM, we find that CBi-CFM attains a lower RCE, which indicates that constraining the prior distribution and the velocity field improves compliance with the conservation law. At the same time, CBi-CFM is slightly worse than Bi-CFM on distribution-based distances. This suggests a tradeoff between enforcing conservation and retaining distributional accuracy, likely because the conservation constraints reduce model expressiveness.

Then, in Fig. \ref{fig:planet}\textbf{c}'s two subplots, to assess whether each reconstruction preserves the conserved energy of the scattering system, we compare the energy of the inferred initial state with the energy of the target final state. The empirical cumulative distribution of the log-scaled Relative Conservation Error summarizes this comparison over the full test set. Curves that rise more rapidly toward one at smaller error values indicate that a larger fraction of samples have low energy mismatch. Since the ground-truth trajectories are produced by numerical integration, their conservation error is not exactly zero because of finite-step and floating-point errors. We observe that CBi-CFM preserves energy at a level nearly comparable to this numerical ground-truth error. Bi-CFM achieves the next best performance, followed by Backbone, whereas Random exhibits the largest deviation. On the other hand, the calibration plots provide a complementary sample-wise view. Each point corresponds to one test trajectory, with the horizontal axis showing the target final-state energy and the vertical axis showing the inferred initial-state energy. The dashed diagonal denotes perfect energy agreement. It can be seen that CBi-CFM still performs comparably to the ground truth. The points of CBi-CFM lie closest to the diagonal, followed by Bi-CFM and then Backbone, whereas the points of Random are barely aligned with the diagonal.

Next, in Fig. \ref{fig:planet}\textbf{d}, we visualize the distributions of representative features of the inferred initial states for the first body under different methods. We separately consider the distributions over all trajectories and over trajectories that contain scattering events. First, the ground truth distributions for all trajectories and those containing scattering differ across several features, which are selected to be presented in the figure. We observe that CBi-CFM, Bi-CFM, and Backbone all distinguish between these two types of trajectories. Second, for each distribution, the closest match to the target distribution is achieved by either Bi-CFM or CBi-CFM, demonstrating their superior ability to capture the dynamics over Backbone. In particular, for the important feature eccentricity, although all methods struggle to fit the distribution, the distribution generated by Backbone is noticeably narrower than Bi-CFM and CBi-CFM over both types of trajectories, revealing the limitation of deterministic models.

Finally, in Fig. \ref{fig:planet}\textbf{e}, we examine whether CBi-CFM can correctly infer the corresponding interval of the initial semi-major axis based on the final state. The results of Bi-CFM and Backbone are recorded in Fig. C10 of Supplementary Note \textcolor{blue}{C}. Specifically, we divide all trajectories into three groups according to their initial semi-major axis: small ($a_0<0.05$), medium ($0.1<a_0<0.2$), and large ($a_0>0.25$). The results demonstrate that CBi-CFM infers the correct intervals of $a_0$, and the inferred distributions closely match the ground truth, which shows that CBi-CFM can capture the causal relationships within the dynamical system.

\subsection{Evaluation on Real-world Million-body Globular Cluster Systems}\label{sec:gc}

To test Bi-CFM in a real observational setting, we apply it to the inverse problem of long-term globular-cluster (GC) evolution. Milky Way GCs are ancient, dense, self-gravitating stellar systems, typically containing $10^4$--$10^6$ stars and evolving over timescales of order $\sim 10~\mathrm{Gyr}$ or longer \citep{Harris1996CatalogGC,Heggie2003GravitationalMB,VandenBerg2013AgesGC}. As tracers of the early assembly history of the Milky Way, they provide a natural testbed for Galactic archaeology \citep{Massari2019OriginGC}. Given the observed surface brightness profiles (SBPs), velocity dispersion profiles (VDPs), and ages of present-day clusters, the goal is to infer initial dynamical states whose forward evolution reproduces the observations.

This inverse problem is challenging because $\sim 10$ Gyr of stellar evolution, two-body relaxation, mass segregation, tidal stripping, binary interactions, and possible core collapse can erase or strongly reshape the memory of the birth state, leaving a degenerate mapping from initial cluster parameters to present-day observables \citep{VesperiniHeggie1997IMF,Baumgardt2003DynamicalEvolution,Giersz2011MonteCarlo47Tuc,Pijloo2015InitialConditions}. Traditional direct and Monte Carlo simulations have been central tools for modelling GC evolution, but brute-force inverse inference over large initial-condition spaces remains computationally prohibitive \citep{Henon1971MCModel,Giersz2013MOCCA,Rodriguez2016CMC,Kremer2020CMCCatalog}. This setting therefore provides a stringent real-world test for a distributional, simulation-trained approach such as Bi-CFM. Additional astrophysical background and computational details are provided in Supplementary Note \textcolor{blue}{A}.

\begin{figure}[t]
    \centering
    \includegraphics[width=0.94\linewidth, trim=14 10 6 0, clip]{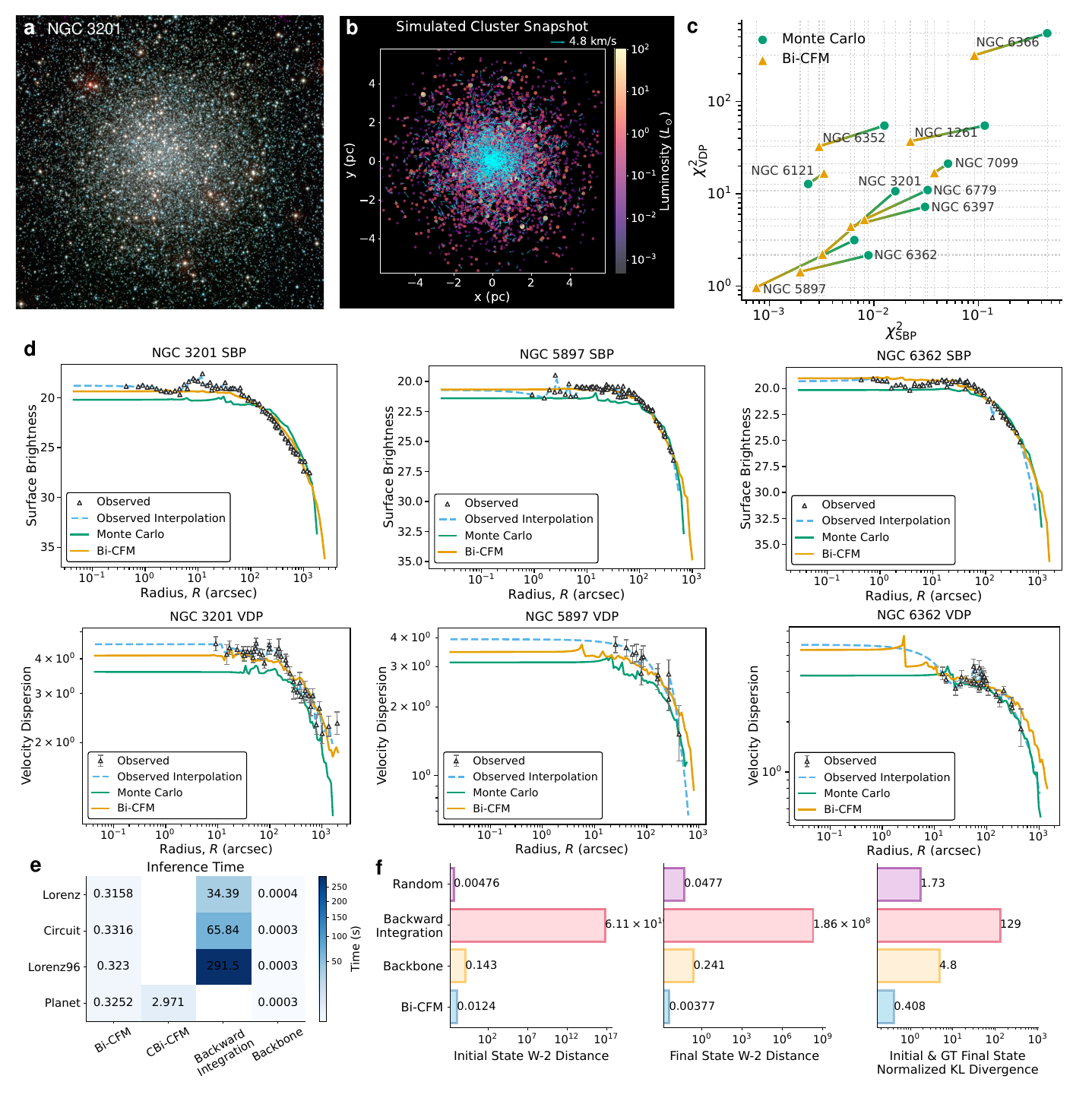}  
    \caption{\textbf{Observational data, simulation data, evaluation metrics, and profiles on million-body globular cluster systems, and ablation analysis of inference time and robustness.} 
    \textbf{a.} The visualization of observational data.
    \textbf{b.} The visualizations of simulation data.
    \textbf{c.} SBP and VDP errors of the Monte Carlo method and Bi-CFM.
    \textbf{d.} Surface brightness profiles and velocity dispersion profiles of observational data, Monte Carlo method, and Bi-CFM.
    \textbf{e.} Inference time of different methods on different systems.
    \textbf{f.} Evaluation metrics of different methods on noisy data of the Lorenz system.}
    \label{fig:cluster}
\end{figure}

\textbf{Experiment setups.} 
We consider a set of observed Milky Way GCs identified by their New General Catalogue (NGC) designations. We show the observational data of NGC 3201 in Fig. \ref{fig:cluster}\textbf{a}. The observation data are constructed from their present-day structural and kinematic properties reported in the literature, which encode information about the long-term dynamical evolution. In this work, the observational constraints are based primarily on projected SBPs, VDPs, and ages. 
Together, these observables provide complementary constraints on each cluster: the SBP traces the projected distribution of stellar light, the VDP constrains the internal kinematics and dynamical mass distribution, and the age sets the evolutionary time at which the simulated cluster should be compared with the observed system. More details of the data are provided in Section \ref{sec:exp_setting}.

\textbf{Baselines and metrics.} Based on these fundamental experimental settings, we compare our method with the state-of-the-art Monte Carlo method \citep{Rui2021MatchGC}. Since the simulations are performed with $N=4\times10^{5}$ particles, we follow the inference in \citet{Rui2021MatchGC} and conduct the baseline comparison on the ten clusters whose initial particle numbers are most likely to be $N=4\times10^{5}$. 
To quantify the agreement between simulated and observed globular-cluster profiles, we compare the simulated and observed profiles using uncertainty-weighted $\chi^2$ statistics for the SBP and VDP following previous works \citep{2008MNRASHeggie, 2014MNRASHeggie, Rui2021MatchGC}. 
These quantities are reported as the SBP and VDP profile errors (See details in Section \ref{sec.metric}).

\textbf{Results.} Overall, we perform simulations from the inferred initial states of different methods, and compare the resulting final states with the observational data using $\chi^2$ errors. For some simulations, the cluster disrupts before reaching the age reported by the literature. In such cases, we use the final simulated state to compute the error.

The observed SBPs and VDPs are measured at cluster-dependent radii, whereas the simulation profiles are represented on fixed radial grids. We therefore interpolate the observational profiles and evaluate them on the same radial grids as the training data, so that the observed clusters can be used as inputs to the model. However, the observed radii are often concentrated within limited radial ranges and do not uniformly cover the full simulation grid. As a result, interpolation can be unreliable in sparsely observed regions, which may introduce inaccurate profile values into the inference procedure. To reduce the influence of poorly constrained radial regions, we construct three radial-coverage variants of the training data. Together with two choices of simulation data sources, this gives six training datasets in total. For each cluster, we report the best-performing Bi-CFM result among the corresponding six trained models. We find that the half-grid variants generally achieve the best overall performance, likely because this radial range better matches the region where the observed radii are most densely sampled (See more details in Section \ref{sec:exp_setting}).

We summarize the results of the baseline and Bi-CFM on the clusters in Fig. \ref{fig:cluster}\textbf{c} with the results for the same cluster connected by line segments. Points closer to the lower-left corner indicate smaller errors in both SBP and VDP. It can be seen that Bi-CFM generally outperforms the baseline Monte Carlo method, except for one cluster where the two methods achieve comparable performance.
To provide a more detailed comparison of the performance, in Fig. \ref{fig:cluster}\textbf{d}, we visualize the profiles of three clusters, including NGC 3201, which corresponds to Fig. \ref{fig:cluster}\textbf{a}, as well as NGC 5897 and NGC 6362. The triangles denote the observational data points, and the vertical intervals indicate their uncertainties. The dashed curves represent the profiles obtained by linear interpolation of the observational data points. The green and yellow curves correspond to the final states evolved from the initial states inferred by the Monte Carlo method and Bi-CFM, respectively. As shown in the figure, the results obtained by Bi-CFM fit the observational data points more closely and better agree with the profiles obtained by linear interpolation.
We can observe that Bi-CFM infers physically more consistent initial conditions whose long-term evolution better matches the observations. Because, unlike the Monte Carlo method, which searches over the initial-condition space through a finite set of forward simulations, Bi-CFM learns the distribution-level physical correspondence between initial and final states. Its success on this long-timescale problem further suggests that the proposed end-to-end bidirectional design effectively mitigates the error accumulation inherent to chaotic dynamics.
This experiment demonstrates that our proposed method can serve as a new tool for inverse problems in GCs, a large-scale, chaotic, and long-timescale system, outperforming traditional approaches.
 
\subsection{Analysis of Inference Time and Robustness}
\label{sec:ablation}

\textbf{Inference time.} In Fig. \ref{fig:cluster}\textbf{e}, we record different methods' inference times for one sample. The horizon $T$ of the Lorenz, Circuit, and Lorenz 96 systems is 3, 1.5, and 7, respectively. Nevertheless, the inference time of Bi-CFM on the GC task is comparable to that observed in the other systems. The times are calculated as the average of 50 batches with batch size $=1$ after warm-up. The results of deep-learning methods are conducted on GPUs, while the Backward Integration is on CPUs because the implementation is not applicable on GPUs. We note that the development of traditional numerical solvers on GPUs can accelerate the process of Backward Implementation in the future. From Fig. \ref{fig:cluster}\textbf{e}, we observe that compared with Backward Integration, CFM achieves a speedup of more than two orders of magnitude. But it is slower than Backbone because the sampling process takes several sampling steps. Also, the time of Backward Integration increases a lot as the number of dimensions increases, while the deep learning methods' inference times are almost consistent. This demonstrates deep learning methods' potential to deal with high-dimensional problems. Additionally, CBi-CFM takes a longer time compared with Bi-CFM, which is mainly because the sampling from the prior distribution involves random walks and the projection, which need additional matrix multiplications.

\textbf{Robustness evaluation.} On the Lorenz system, we further evaluate the robustness of different methods under noisy observations. Specifically, to mimic uncertainty in the observed final states, we add Gaussian noise with a scale of $0.1$ to the target final states at $T=2.4$ and re-evaluate all distribution-level metrics. Three representative metrics are reported in Fig.~\ref{fig:cluster}\textbf{f}, while the remaining two metrics are provided in Supplementary Note~\textcolor{blue}{C}. The results show that Bi-CFM remains robust under noisy targets and consistently achieves the best performance among all baselines. In contrast, Backward Integration exhibits the largest performance degradation after noise is added, indicating its strong sensitivity to perturbations in the final state. This is consistent with the nature of chaotic inverse problems, where small observational errors can be amplified during reverse-time integration, whereas the end-to-end distributional modelling of Bi-CFM mitigates such error accumulation.

\section{Discussion} \label{sec.diss}

In this work, we propose Bi-CFM, a probabilistic and physically grounded approach for solving the inverse problems of chaotic systems. The extensive experiments demonstrate its effectiveness and reveal the potential of deep generative modeling in tackling the ill-posed inverse problems.
On the one hand, Bi-CFM tackles one of the central challenges in chaotic systems, the exponential accumulation of errors over time, through end-to-end learning and bidirectional modelling. On the other hand, its probabilistic formulation captures the stochasticity and non-uniqueness of chaotic evolution. Its conservation-constrained extension, CBi-CFM, further incorporates a key physical principle by enforcing conserved quantities during inverse modelling. In addition, as the system dimension increases, the inference time of the deep learning-based methods remains nearly unchanged, suggesting their potential to mitigate the curse of dimensionality \citep{bellman1966dynamic}. Together, Bi-CFM and CBi-CFM provide a potential route to high-dimensional, long-timescale inverse problems in Galactic archaeology, planetary dynamics, and other chaotic physical systems.

In a series of systematic experiments, we observe several interesting phenomena that could provide insights for future work. Firstly, we found that some initial states $\mathbf{u}_0$ generated by Backward Integration and their corresponding ground truth $\mathbf{u}_0^\text{gt}$ from the dataset differ significantly in magnitude, yet the final states obtained through forward evolution $g$ are close. This reflects the ill-posedness and multi-solvability of inverse problems and suggests that there are reasonable solutions far beyond the data distribution, which could lead to important scientific discoveries. However, to enable deep learning methods to uncover these solutions, enhancing generalization is necessary, which can be achieved by incorporating physical constraints and priors.
Secondly, we observe a degradation when conservation constraints are imposed in CBi-CFM. This tradeoff reflects the inherent tension between enforcing physical laws and maintaining representational flexibility, particularly in high-dimensional chaotic systems. We anticipate that future theoretical developments will offer understandings and solutions to this phenomenon. Also, future work may explore adaptive or learnable conservation constraints, where the strength of the physical constraint can be dynamically adjusted during training. Such an approach could balance physical consistency and expressiveness according to the requirements and specific settings.

In addition, there are several possible improvements worth studying in future work. Firstly, we find that Bi-CFM maintains an advantage in modeling the initial-state distribution compared with the final-state, which can be attributed to its direct optimization objective focusing on the initial distribution. A potential improvement could be achieved by introducing additional constraints on the final states during training, for example, by evolving the inferred initial states through the learned dynamics. Secondly, although Bi-CFM achieves a hundredfold acceleration over traditional backward-integration-based numerical solvers, it is still slower than the base model Backbone. This is mainly because the sampling process in CFM involves multiple iterative steps. However, this limitation can be alleviated by adopting improved flow matching approaches, such as one-step methods, such as Short-cut Models \cite{fransone} or MeanFlow \cite{meanflow}. Another possible solution is to distill the current model into a smaller and more efficient version \cite{song2023consistency, meng2023distillation}, thereby reducing the sampling steps and accelerating the inference of each sampling step. Finally, another interesting research direction is to provide explicit confidence estimates for the inferred states. In future work, we can compute the probability estimates by solving the probability-flow equation with numerical solvers. This enhancement could make the model more interpretable and reliable. In the future, equipping generative frameworks with confidence estimates may unlock potential for applying generative learning to chaotic physical systems. We aim to systematically investigate these directions in our future study.

\section{Methods} \label{sec.meth}

\subsection{Conditional Flow Matching}\label{sec.cfm}

Conditional Flow Matching (CFM) \cite{lipmanflow, liuflow} has emerged as a state-of-the-art framework for generative modeling. To sample from the \emph{target} probability distribution, the CFM framework constructs a probability ``flow'' path $p_{\tau}$ that starts from an easy-to-sample \emph{source} distribution $p_0$ (\textit{e.g.}, a Gaussian distribution) at $\tau=0$, and ends at the target distribution $p_1$ at $\tau=1$. 
The goal of CFM is to learn the velocity field $v^{\tau}$ along this flow path, so that samples $\mathbf{x}^0$ drawn from $p_0$ can be transformed into samples $\mathbf{x}^1$ from the target distribution by integrating the following ODE
\begin{align}\label{eq:sampling}
    \frac{d\mathbf{x}^\tau}{d\tau} = v^\tau(\mathbf{x}^\tau)
\end{align}
from $\tau=0$ to $\tau=1$.
Consequently, the model $v_\theta(\tau,\mathbf{x}^\tau)$ needs to regress to the velocity using the \emph{marginal} loss 
\begin{equation*}
    \mathcal{L}_{\text{marginal}} \coloneqq \mathbb{E}_{\tau\sim p(\tau),\mathbf{x}^\tau\sim p_\tau(\mathbf{x}^\tau)}\left[\| v_\theta(\tau,\mathbf{x}^\tau) - v^\tau(\mathbf{x}^\tau) \|_2^2\right].
\end{equation*}

This marginal velocity field in the above equation is the expectation of conditional velocities $v^\tau(\mathbf{x}^\tau|\mathbf{z})$ as 
\begin{equation*}
v^\tau(\mathbf{x}^\tau)=\mathbb{E}_{\mathbf{z}\sim p(\mathbf{z}|\mathbf{x}^\tau)}[v^\tau(\mathbf{x}^\tau|\mathbf{z})],
\end{equation*}
where $\mathbf{z}$ is the condition and $p_\tau(\mathbf{x}^\tau|\mathbf{z})$ is a conditional probability path.
The conditional probability path is typically much simpler than the marginal one $p_t(\mathbf{x}_t)$, because $\mathbf{z}$ usually includes the start or end point of the path. Conditioned on the known $\mathbf{z}$, one can construct conditional probability paths without high-dimensional integration introduced by the expectation operation. As a result, compared with the marginal one, the conditional velocity $v_{t}(\mathbf{x}_t|\mathbf{z})$ is tractable, which leads to the tractable \emph{conditional} training loss 
\begin{equation*}
\mathcal{L}_{\text{conditional}} \coloneqq \mathbb{E}_{\tau\sim p(\tau),\mathbf{x}^\tau\sim p_\tau(\mathbf{x}^\tau), \mathbf{z}\sim p(\mathbf{z}|\mathbf{x}^\tau)}\left[\| v_\theta(\tau,\mathbf{x}^\tau) - v^\tau(\mathbf{x}^\tau|\mathbf{z})\|_2^2\right].
\end{equation*}
Additionally, this conditional loss is proven to be equivalent to the \emph{marginal} training loss when used to optimize the model parameters $\theta$ \cite{lipmanflow, liuflow}, \emph{i.e.},$\nabla_\theta \mathcal{L}_{\text{conditional}} = \nabla_\theta \mathcal{L}_{\text{marginal}}$, so we can take the conditional loss as the training objective.
Compared to other flow-based generative models, such as the normalizing flow \cite{rezende2015variational}, the conditional loss grants CFM the effectiveness of training and scalability to complex probability distributions and large datasets. Consequently, we adopt it to model the distribution of chaotic systems.

In this paper, we choose $p_0=\mathcal{N}(0;I)$ and $\mathbf{z}=(\mathbf{x}^0,\mathbf{x}^1)$ pair. In this case, the constructed conditional probability path is degraded into effectively an interpolation between two samples: $\mathbf{x}^\tau = \alpha(\tau) \mathbf{x}^0 + \beta(\tau) \mathbf{x}^1 + \eta(\tau) \boldsymbol{\epsilon}$ where $\alpha(\tau), \beta(\tau)$ are the time-dependent interpolation coefficients, while $\eta(\tau)$ is a small random noise scale to maintain stochasticity. The source and target samples can be coupled in this case, allowing more flexible flow mappings. %
As for the sampling process, since the sampling in CFM involves simulating an ODE (Eq. \ref{eq:sampling}), different numerical solvers can be used. In this study, we adopt the Dormand–Prince (RKDP) method \cite{DORMAND198019} with 100 steps.

\subsection{Bidirectional Modeling}

The bidirectional modeling strategy is an approach we introduce to learn both forward and reverse dynamics of chaotic systems jointly, which is motivated by the error accumulation in chaotic systems. Empirically, in the experiments, we observe that although the model can accurately approximate the initial state distribution $p(\mathbf{u}_0^\text{gt})\ (\mathbf{u}_0^\text{gt}\in \mathbb{R}^n)$, the inferred initial states $\mathbf{u}_0$ often exhibit amplified errors when evolved to $g(\mathbf{u}_0, T)$, resulting in large deviations in the final states. The bidirectional formulation mitigates this issue by improving consistency between forward and reverse dynamics, thereby enabling the model to learn a more complete and physically consistent representation of the underlying chaotic dynamics. 

Specifically, our goal is to learn both the reverse dynamics represented by the conditional distribution \( p(\mathbf{u}_0^\text{gt} \mid \mathbf{u}_T^\text{gt}) \) and the forward dynamics corresponding to \( p(\mathbf{u}_T^\text{gt} \mid \mathbf{u}_0^\text{gt}) \). In practice, we let the model learn to generate pairs $(\mathbf{u}_T^{\text{gt}}, \mathbf{u}_0^{\text{gt}})$ in the space $\mathbb{R}^{2n}$. During training, we randomly select a portion of sample pairs where the positions of $\mathbf{u}_T^{\text{gt}}$ are used as conditions, \textit{i.e.}, noise-free variables, while the corresponding $\mathbf{u}_0^{\text{gt}}$ positions are still trained to learn how to denoise from Gaussian noise towards the ground truth $\mathbf{u}_0^{\text{gt}}$ with conditioned $v_\theta(\tau, \cdot, \mathbf{u}_T^{\text{gt}})$. The corresponding time pair of these samples is $(\tau_1, \tau_2)=(\tau, 1)$, where $\tau_2=1$ denotes that the $\mathbf{u}_T^{\text{gt}}$ position is noise-free. For the remaining samples, the roles are reversed: $\mathbf{u}_0^{\text{gt}}$ serves as the condition, and $\mathbf{u}_T^{\text{gt}}$ is trained to denoise with $v_\theta(\tau, \mathbf{u}_0^{\text{gt}}, \cdot)$.  The corresponding time pair of these samples is $(\tau_1, \tau_2)=(1, \tau)$. The details of the conditional CFM and the implementation of bidirectional training are reported in Supplementary Note \textcolor{blue}{B}.

During sampling, our target is to sample from $p(\mathbf{u}_0 \mid \mathbf{u}_T^\text{gt})$. Thus, we set $\tau_2=1$ and $\tau_1$ gradually increasing from 0 to 1 with probability flowing from $p(\mathcal{N}(0; I) \mid \mathbf{u}_T^\text{gt})$ to $p(\mathbf{u}_0 \mid \mathbf{u}_T^\text{gt})$. Through this bidirectional sampling and training scheme, the model jointly learns both forward and reverse dynamics, enabling a consistent representation of the system's evolution across directions.

\subsection{Conservation Law Constraint} \label{sec.conserve}

Conservation laws play an important role in physics by constraining the evolution of a system and ensuring the consistency of its dynamics. According to Noether's theorem \cite{noether1983invariante}, every continuous symmetry of a physical system corresponds to a conserved quantity: translational symmetry leads to conservation of momentum, rotational symmetry to conservation of angular momentum, and time invariance to conservation of energy. These conservation laws not only simplify complex dynamics but also encode causal relationships between symmetry and physical invariants. Enforcing such constraints in data-driven models guides the learned dynamics to remain physically consistent. As a result, we propose an approach to constrain Bi-CFM to follow the conservation laws, which can be applied to any system in which the conserved quantities can be computed from the observed states.

To infer an initial state \(\mathbf{u}_0\) that satisfies the conservation law, we constrain \(p(\mathbf{u}_0)\) to lie on the \textit{conservation manifold}, the data manifold constrained by the conservation law. This is achieved by imposing conservation constraints on both the prior distribution and the velocity field as illustrated in Fig. \ref{fig:fig1}\textbf{b}. We first constrain the prior distribution to lie within the conservation-constrained manifold, and then restrict the velocity field to remain tangent to this manifold. In doing so, the entire probability flow is constrained to evolve within the conservation manifold.

Mathematically, suppose the system admits a conserved quantity \( H \). For example, in an $N$-body planet-star system, the total energy is conserved and can be written as
\[
H = -\sum_{i=1}^{N} \frac{G M m_j}{2a_i},
\]
where \( m_i \) and \( a_i \) denote the mass and semi-major axis of the \( i \)-th body, respectively, and \( M \) is the mass of the central star. 
For a trajectory \((\mathbf{u}_0, \mathbf{u}_1, \ldots, \mathbf{u}_T)\), the conservation law requires that $H(\mathbf{u}_t)$ remains constant for all \( t = 0, \ldots, T \). 
Therefore, for a given target final state \(\mathbf{u}_T^\text{gt}\), our goal is to infer an initial state \(\mathbf{u}_0\) such that $H(\mathbf{u}_0) = H(\mathbf{u}_T^\text{gt})$, ensuring that the inferred initial condition lies on the same conservation manifold as the target state. Since our objective is to ensure that the generated \(\mathbf{u}_0\) satisfies the conservation law, all the following operations are applied to the \(\mathbf{u}_0\) component within the bidirectional modeling framework. In the following descriptions, we also omit the notation denoting the condition \(\mathbf{u}_T^\text{gt}\) for brevity.

\textbf{Conservation-constrained prior distribution.} To sample within the conservation manifold of $\mathbf{u}$ defined by \( H(\mathbf{u}) = H(\mathbf{u}_T) \), we perform random walks along the tangent space of the conservation manifold starting from \(\mathbf{u}_T\), and finally apply the projection correction to mitigate numerical errors. 
Specifically, we initialize \(\mathbf{u} = \mathbf{u}_T^\text{gt}\) and iteratively update as follows: 
at each step, we sample a random vector \(\mathbf{r}\) from the Gaussian distribution, project it onto the tangent space of the conservation manifold via \(\mathbf{r} \cdot P(\mathbf{u})\), and update the current state by $\mathbf{u}' = \mathbf{u} + \mathbf{r}\cdot P(\mathbf{u}) \Delta r$, where \(\Delta r\) denotes the step size. 
The projection matrix \(P(\mathbf{u})\) associated with the conserved quantity \(H\) is defined as
\begin{equation}
    P(\mathbf{u}) = I - \frac{\nabla H(\mathbf{u}) \nabla H(\mathbf{u})^\top}{\|\nabla H(\mathbf{u})\|^2}.
    \label{eq:proj}
\end{equation}
We provide a proposition to guarantee that $\mathbf{r} \cdot P(\mathbf{u})$ is in the tangent space of the conservation manifold, with its proof included in Supplementary Note \textcolor{blue}{D}.
\begin{proposition}
Let $H:\mathbb{R}^n\to\mathbb{R}$ be continuously differentiable, and let the conservation manifold be defined as
\[
\mathcal{M} = \{\mathbf{u}\in\mathbb{R}^n : H(\mathbf{u}) = H(\mathbf{u}_T^\text{gt})\},
\]
with $\nabla H(\mathbf{u}) \neq \mathbf{0}$. With $P$ defined as Eq.~\ref{eq:proj}, for any $\mathbf{r}\in\mathbb{R}^n$, the vector $\mathbf{r}\cdot P(\mathbf{u})$ lies in the tangent space $T_{\mathbf{u}}\mathcal{M}$ of the manifold at $\mathbf{u}$.
\label{theo:conservation}
\end{proposition}

Although the proposition guarantees that \(\mathbf{r}\cdot P(\mathbf{u})\) lies in the tangent space, we observe a tradeoff between accuracy and efficiency when updating with $\mathbf{u}' = \mathbf{u} + \mathbf{r}\Delta r$.
If \(\Delta r\) is large, the update introduces a nonnegligible numerical error that pushes \(\mathbf{u}\) away from the conservation manifold, and the deviation accumulates over iterations. If \(\Delta r\) is too small, the random walk remains confined to a small neighborhood around \(\mathbf{u}_T\), which limits the diversity of the prior distribution and the expressiveness of the algorithm. As a result, we mitigate this tradeoff by combining the projection correction. The projection step is also iterative. At each iteration, the goal is to correct the current state \(\mathbf{u}\) so that its conserved quantity \(H(\mathbf{u})\) moves closer to the target value \(H(\mathbf{u}_T)\). This correction is performed along the gradient direction of \(H\), leading to the following update rule:
\begin{align*}
    \mathbf{u}' = \mathbf{u} - \frac{H(\mathbf{u}) - H(\mathbf{u}_T)}{\|\nabla H(\mathbf{u})\|^2} \nabla H(\mathbf{u}).
\end{align*}
Each iteration reduces the deviation $\mid H(\mathbf{u}') - H(\mathbf{u}_T)\mid$
and gradually brings $\mathbf{u}$ back to the conservation manifold. 
In practice, a few iterations are typically sufficient to achieve a reduction in numerical error.

\textbf{Conservation-constrained velocity field.} To ensure that the learned flow remains within the conservation manifold, we further constrain the model-predicted velocity field to lie in the tangent space of the conservation manifold. 
Specifically, this is achieved by projecting the velocity field onto the tangent space. 
In both the training and sampling process, we multiply the velocity field with the projection matrix \( P(\mathbf{u}) \) in Eq. \ref{eq:proj}. Theoretical justification for this operation is provided by Proposition \ref{theo:conservation}, which guarantees that \( v_\theta \cdot P(\mathbf{u}) \) always resides within the tangent space of the conservation manifold. Additionally, during training, the loss $\mathcal{L}_\text{conditional}$ of CFM is defined on conditional velocities $v_t(\mathbf{u}\mid\mathbf{z})$, so we can only enforce the constraint on these conditional velocities. We point out that the marginal velocity, as the expectation of conditional velocities, remains in the tangent space as well, since both the expectation operator and the tangent space are linear.

\subsection{Evaluation Metrics}\label{sec.metric}

As described in Section~\ref{sec.bicfm}, we evaluate the performance of different methods using distribution-level distances. Specifically, we employ two classic and widely used metrics: the Wasserstein distance and the KL divergence. Additionally, for the globular cluster systems, we use $\chi^2$ statistics for SBPs and VDPs as evaluation metrics.

\textbf{Wasserstein distance.} Formally, for two probability distributions \( q_1 \) and \( q_2 \) defined on a metric space with cost function \( d(\cdot, \cdot) \), the Wasserstein distance (W-$p$ distance) is defined as:
\begin{equation*}
    W_p(q_1, q_2) = \left( \inf_{\gamma \in \Pi(q_1, q_2)} \mathbb{E}_{(x, y) \sim \gamma} [ \frac{1}{p} d(x, y)^p ] \right)^{1/p},
\end{equation*}
where \( \Pi(q_1, q_2) \) denotes the set of all joint distributions with marginals \( q_1 \) and \( q_2 \). In our implementation, we adopt the case of \(p = 2\), corresponding to the squared Euclidean cost \(d(x, y) = \|x - y\|_2\). The Wasserstein distance measures the minimal cost of transporting one probability distribution to another, providing a meaningful metric that captures both the geometry and global structure of distributions. 

For numerical stability and efficiency, we compute the W-2 distance using an entropy-regularized approximation, namely the Sinkhorn divergence \cite{cuturi2013sinkhorn}, with the package \texttt{GeomLoss} \cite{feydy2019interpolating}. Sinkhorn divergence adds an entropic regularization term to the original optimal transport problem, making it convex and efficiently solvable via the Sinkhorn-Knopp algorithm \cite{sinkhorn1964relationship}. In \texttt{GeomLoss}, the entropy regularization term is controlled by the parameter \texttt{blur}. In our experiments, we set \texttt{blur} \(= 0.01\), which provides a stable and differentiable estimation of the W-2 distance.

\textbf{KL divergence.} The KL divergence quantifies the relative entropy between two distributions, reflecting how much information is lost when \( q \) is used to approximate \( p \). It is defined as:
\begin{equation}
    D_{\mathrm{KL}}(p \parallel q) = \int p(x) \log \frac{p(x)}{q(x)} dx.
\end{equation}
Unlike the W-2 distance, it is asymmetric and does not satisfy the triangle inequality, but it provides a sensitive measure of local mismatches between distributions. 

We estimate the KL divergence between two empirical distributions using the $k$-nearest-neighbor-based estimator \cite{wang2009divergence}. This nonparametric method estimates local probability densities from the distances to the $k$ nearest neighbors of each sample and computes divergence directly from these density ratios. In our implementation, we set $k=5$ and use 2500 samples for each distribution pair. To quantify uncertainty, we perform 100 bootstrap resamplings of the data and recompute the divergence estimates for each trial. This procedure yields a robust, sample-efficient estimate of KL divergence that is suitable for comparing high-dimensional dynamical distributions.

\textbf{Normalization.} For the distribution-level metrics over states at different time steps, to prevent any distribution at a single time step from dominating the overall distance, we normalize the distributions of different steps by subtracting the mean and dividing by the standard deviation.

\textbf{$\chi^2$ statistics.} In globular cluster systems, we use uncertainty-weighted $\chi^2$ statistics following previous works \citep{2008MNRASHeggie, 2014MNRASHeggie, 2018ApJK, 2019ApJY, Rui2021MatchGC}. 
We estimate an effective uncertainty from the catalogue quality weights \citep{Trager1995SBP}. For the $i$-th SBP point, we define
\begin{equation}
    \delta \mu_{V,i} = \frac{\delta \mu_{V,0}}{w_i},
\end{equation}
where $w_i$ is the catalogue quality weight. The normalization is estimated from the residuals of the tabulated third-order Chebyshev fit,
\begin{equation}
    \delta \mu_{V,0}
    =
    \left[
    \frac{1}{N_{\rm SBP}}
    \sum_i w_i^2 \epsilon_{\mu,i}^2
    \right]^{1/2},
\end{equation}
where $\epsilon_{\mu,i}$ is the residual of the observed SBP point relative to the Chebyshev fit, and $N_{\rm SBP}$ is the number of valid SBP observational radial points. Points with non-positive weights are excluded. The resulting inverse-variance weights are then normalized before constructing the interpolated observational profile.

For the VDP data, we use the reported upper and lower observational uncertainties from the radial-velocity or proper-motion measurements in the Baumgardt compilation \citep{Baumgardt2018VDP,Baumgardt2019GaiaVDP}. Because the profile statistic uses a single symmetric weight per point, we define an effective uncertainty
\begin{equation}
    \delta \sigma_i =
    \frac{\delta \sigma_{i,+} + \delta \sigma_{i,-}}{2}.
\end{equation}
This gives the inverse-uncertainty weight used in the observational profile construction.

The discrepancies between a simulated profile and the corresponding observed profile are evaluated at the observed radii. The simulated profile is linearly interpolated to each observed radius, and we compute
\begin{equation}
    \chi^2_{\rm SBP}
    =
    \frac{1}{N_{\rm SBP}}
    \sum_i
    W^{\rm SBP}_i
    \left[
    \mu_{V,\rm sim}(R_i)-\mu_{V,\rm obs}(R_i)
    \right]^2,
\end{equation}
and
\begin{equation}
    \chi^2_{\rm VDP}
    =
    \frac{1}{N_{\rm VDP}}
    \sum_i
    \left[
    \frac{
    \sigma_{\rm sim}(R_i)-\sigma_{\rm obs}(R_i)
    }{
    \delta \sigma_i
    }
    \right]^2,
\end{equation}
where $W^{\rm SBP}_i=\frac{\omega_i}{\sum_j \omega_j},$ denotes the normalized SBP weight.

\subsection{Experiment Setups} \label{sec:exp_setting}

\textbf{Three classic chaotic systems.} First, we consider the Lorenz system, a canonical three-dimensional chaotic system originally introduced as a simplified model of atmospheric convection \citep{lorenz2017deterministic}. We use the standard chaotic parameter setting, under which the system exhibits the well-known butterfly-shaped strange attractor and strong sensitivity to initial conditions. Second, we study a three-dimensional autonomous chaotic system that has been used in applications such as secure communication and encryption \citep{dadras2009novel}. This system can generate multi-lobe chaotic attractors by varying a single parameter. Finally, we consider the Lorenz 96 system, a high-dimensional chaotic system widely used as a benchmark for nonlinear dynamics and atmospheric modelling \citep{lorenz1996predictability}. It captures complex interactions among multiple variables under periodic boundary conditions and is commonly used to study high-dimensional chaos. We adopt the forcing parameter $F=20$ and set the dimension to $N=10$. More details are provided in Supplementary Note \textcolor{blue}{A}.

\textbf{Three-body planetary systems.} To sample a wide range of initial conditions, we set the mass of the planets ranging from Earth-size to a few Jupiter masses, and we set the initial semi-major axis of the innermost planet to be drawn log-uniformly between 0.02 and 0.4 AU. The semi-major axes of the other two planets were then assigned to allow efficient scattering. %
We initialize the planets on nearly circular and coplanar orbits, as expected from disk-planet interactions during the early stages of system formation.

We use \texttt{REBOUND} to run the N-body scattering simulations \citep{rebound}, with general relativity correction, using the \texttt{gr-potential} option in \texttt{REBOUNDx} \cite{tamayo2020reboundx}. In addition, we assume the bodies collide by conserving mass and momentum, as in the built-in \texttt{REBOUND} \texttt{collision} routine, and we eject (or remove) the planets when they reach the escape distance of 1000 AU. Further details of the system are presented in Supplementary Note \textcolor{blue}{A}.

\textbf{Real-world million-body globular cluster systems.}

We compile observational SBPs, VDPs, and ages from standard Milky Way globular-cluster catalogues. Specifically, the observational data are collected from the observational compilation used in \citep{Rui2021MatchGC}, the SBPs are taken from the catalogue of \citet{Trager1995SBP}, which remains one of the standard large compilations for Milky Way globular cluster structure studies. 
The VDPs are adopted from the 4th version of the GC database assembled by Holger Baumgardt\footnote{https://people.smp.uq.edu.au/HolgerBaumgardt/globular/veldis.html} \citep{Baumgardt2018VDP, Baumgardt2019GaiaVDP, Vasiliev2023GaiaEDR3}, which combines large samples of line of sight velocities and Gaia proper motions to provide kinematic profiles for a substantial fraction of the Milky Way globular cluster population. 
We further use age estimates from the literature, based on colour magnitude diagrams, eclipsing binaries, horizontal branch constraints, and HST photometry \citep{Faria2002NGC6352,Correnti2018NGC6397,Ying2024NGC3201,Kaluzny2013M4,Campos2013NGC6366,VandenBerg2018M55NGC6362,AguadoAgelet2025CARMAII}. 

As for the corresponding simulation data for training, we collect simulation data from the previous study \citep{Kremer2020CMCCatalog} and further generate additional data using the same solver, as shown in Fig. \ref{fig:cluster}\textbf{b}. In detail, we perform a suite of star-cluster simulations using Cluster Monte Carlo (CMC) code\footnote{https://clustermontecarlo.github.io/CMC-COSMIC/index.html}, a state-of-the-art code for simulating collisional star-cluster dynamics \citep{Rodriguez2016CMC,Rodriguez2022CMC}. CMC follows the H\'enon-type Monte Carlo approach, in which two-body relaxation is modeled statistically by representing the cumulative effect of many weak gravitational encounters over a timestep as an effective encounter between neighboring particles \citep{Henon1971MCMethod,Rodriguez2022CMC}. 

To reduce the effect of poorly constrained radial regions, we construct three variants of the training data by evaluating each profile over different radial intervals. Specifically, for each profile, we extract its values on the original full radial grid, on a subset covering approximately one half of the radial range, and on a subset covering approximately one quarter of the radial range. 

As for the physical parameters of simulations, our choices are based on the comparison between the \texttt{CMC Cluster Catalog} \citep{Kremer2020CMCCatalog} and observations by \citet{Rui2021MatchGC}, who fit 59 Milky Way GCs using observed SBPs and VDPs and show that a substantial subset of the catalog models can reproduce present-day Galactic GCs. We therefore choose parameter ranges that broadly cover the cluster-scale properties of most Milky Way GCs, while keeping the simulation suite computationally feasible. Specifically, we fix the initial number of particles to $N=4\times10^{5}$ and vary three cluster-scale parameters: the initial virial radius $r_{\rm v}$, the Galactocentric radius $R_{\rm g}$, and the metallicity. 
All other initial conditions follow \citet{Kremer2020CMCCatalog}. Each model is evolved to 13.8 Gyr unless the cluster disrupts.

\subsection{Baseline Methods}
\label{sec:baseline}

In Section \ref{sec.bicfm} and \ref{sec:planet}, we consider three main baselines. The first is the \textit{Backward Integration} with the traditional numerical algorithm mentioned in Section \ref{sec.intro}. For deterministic, non-chaotic systems, it can infer the initial state with low numerical errors when the time step is sufficiently small. In practice, we observe that the performance first increases and then decreases as the number of time steps grows. When the number of steps is small, the model struggles to converge accurately. When it becomes too large, error accumulation dominates. In our experiments, we select the time step size that yields the best performance, and the detailed results are provided in Supplementary Note \textcolor{blue}{B}. Additionally, we note that this baseline requires explicit knowledge of the governing equations, which is not needed by Bi-CFM. The second is \textit{Backbone}, the base neural architecture of the CFM algorithm. The input of it is the target final states, while the output is the inferred initial states. The third baseline is selecting initial states from the ground truth dataset randomly, which we refer to as \textit{Random} hereafter. The distribution of initial states $\mathbf{u}_0$ obtained from this baseline and the distribution of the corresponding $g(\mathbf{u}_0, T)$ is the same as the ground truth $p(\mathbf{u}_0^\text{gt})$ and $p(\mathbf{u}_T^\text{gt})$, respectively. However, it fails to capture the causal relationships between states at different time points. We report more details of baseline implementation in Supplementary Note \textcolor{blue}{B}.

In Section \ref{sec:gc}, we consider the state-of-the-art Monte Carlo approach as a baseline \cite{Rui2021MatchGC}. The method infers the initial conditions of an observed cluster by searching over a set of forward-evolved cluster models. Each candidate is initialized with a prescribed set of physical parameters and evolved to old ages using cluster simulations. Its projected SBPs and VDPs are then computed and compared with the corresponding observational profiles. The candidate whose final snapshot yields the smallest profile discrepancy is taken as the best-matching evolutionary analogue, and its initial parameters are adopted as the inferred initial conditions of the observed cluster.

\section{Acknowledgements}
We gratefully acknowledge the support of Westlake University Research Center for Industries of the Future; Westlake University Center for High-performance Computing. The content is solely the responsibility of the authors and does not necessarily represent the official views of the funding entities.

\begin{appendices}

\section{Detailed Settings of Systems in Experiments}

We herein introduce the detailed settings of the systems on which we conduct experiments.

\subsection{Lorenz System}

The Lorenz system is a classic example of chaotic dynamical systems, originally studied as a simplified model for atmospheric convection \cite{lorenz2017deterministic}. It consists of a system of three coupled, nonlinear ordinary differential equations (ODEs):
\begin{equation*}
\begin{cases}
\dot{x} = \sigma (y - x), \\
\dot{y} = x (\rho - z) - y, \\
\dot{z} = x y - \beta z,
\end{cases}
\end{equation*}
where \( x(t), y(t), z(t) \) are the state variables, and $\sigma$, $\rho$, and $\beta$ are system parameters controlling the rate of diffusion, buoyancy, and dissipation, respectively. A commonly studied setting is $\sigma = 10$, $\rho = 28$, and $\beta = 8/3$. Under these parameters, the system exhibits the well-known butterfly-shaped strange attractor, characterized by sensitive dependence on initial conditions and a positive largest Lyapunov exponent.

During data generation in this experiment, we generate $50000$ trajectories in total. Among them, 45000 trajectories are used for training and 5000 for testing. We sample initial conditions from the uniform distribution on $[-1,1] \times [0,2] \times [-1,1]$, \textit{i.e.}, $\mathcal{U}([-1,1] \times [0,2] \times [-1,1])$. We take the Runge–Kutta method (RK45) \cite{dormand1980family} with $\Delta t=10^{-5}$ to solve the equation numerically with the package \texttt{scipy} \cite{virtanen2020scipy}. The visualization of the trajectory's $(x,y)$ value is provided in Fig. \ref{fig:lorenz_vis}, which shows the classical two-lobe attractor.

\begin{figure}
    \centering
    \includegraphics[width=0.5\linewidth]{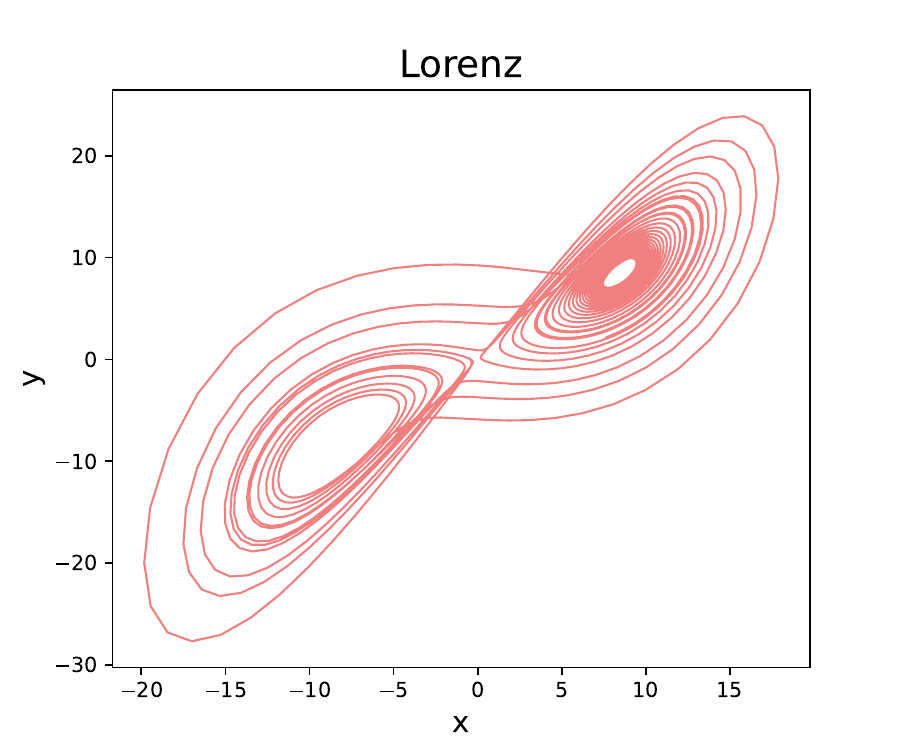}  
    \caption{\textbf{Visualization of one trajectory in the Lorenz system dataset.}}
    \label{fig:lorenz_vis}
\end{figure}

\subsection{Circuit System}

The next system we consider is a three-dimensional autonomous chaotic system, which has implications for practical applications such as communication, encryption technologies, and secure information transmission \cite{dadras2009novel}. By varying a single parameter, this system can generate two-lobe, three-lobe, and four-lobe chaotic attractors. The system is described as follows.
\begin{equation*}
    \begin{cases}
\dot{x} = y - ax + byz, \\
\dot{y} = cy - xz + z, \\
\dot{z} = dxy - hz,
\end{cases}
\end{equation*}
where \( x(t), y(t), z(t) \) are the state variables, and \( a, b, c, d, h \) are positive constant parameters. Here, we choose $a=3, b=2.7, c=1.7, d=2$, and $h=9$.

In this experiment, we use 90000 trajectories for training and 10000 trajectories for evaluation. The initial states of the trajectories are sampled from the Gaussian distribution $\mathcal{N}((1,-3,1); I)$. We adopt an explicit Runge-Kutta method of order 8 DOP853 \cite{hairer1993solving} as the numerical solver, whose implementation is from \texttt{scipy} and time step size $\Delta t = 10^{-5}$. Fig. \ref{fig:circuit_vis} is the $(x,y)$ visualization of this system, demonstrating the three-lobe structure of the attractor.

\begin{figure}
    \centering
    \includegraphics[width=0.5\linewidth]{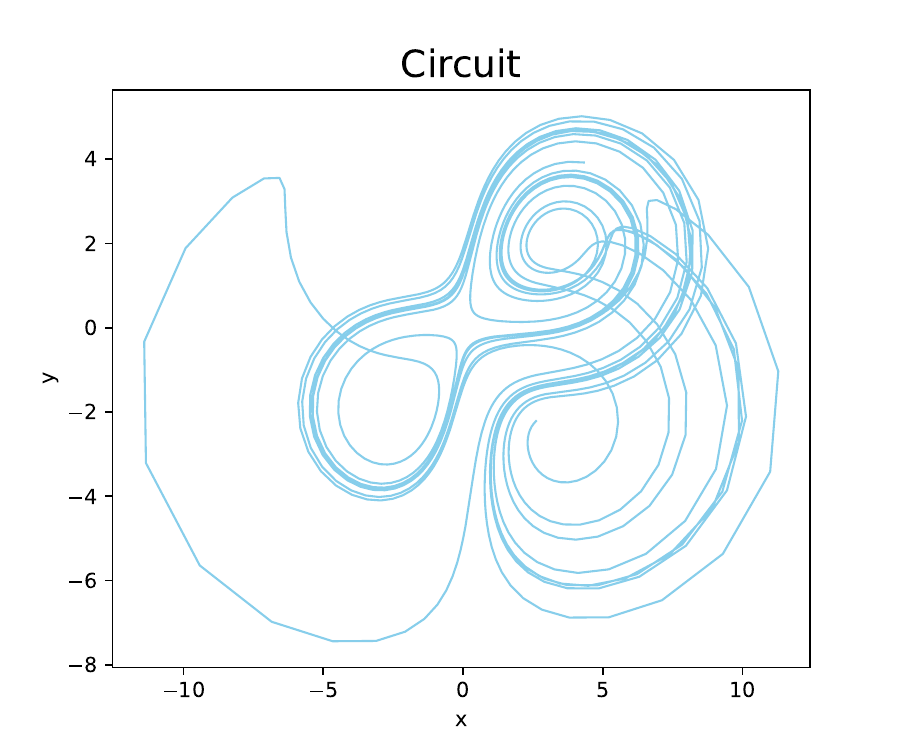}  
    \caption{\textbf{Visualization of one trajectory in the circuit system dataset.}}
    \label{fig:circuit_vis}
\end{figure}

\subsection{Lorenz96 System}

Finally, we consider the high-dimensional Lorenz 96 system, widely studied in nonlinear dynamics and originally proposed to represent the atmospheric convection in large-scale weather models \cite{lorenz1996predictability}. The Lorenz 96 system is particularly useful for studying chaos in large systems and is often used as a benchmark in studies of high-dimensional chaos, where interactions between multiple variables lead to complex, unpredictable dynamics. The system is governed by the following equations.
\begin{align*}
\dot{u}_i = (u_{i+1} - u_{i-2}) u_{i-1} - u_i + F, \quad i = 1, 2, \dots, N,
\end{align*}
where \(u_i\) represents the state at the \(i\)-th variable, and \(N\) is the number of variables in the system. The parameter \(F\) is a forcing term that controls the level of external forcing or input, and the system is typically considered in a periodic boundary condition, \textit{i.e.}, \(u_0 = u_N\) and \(u_{N+1} = u_1\). In our experiments, we adopt $F = 20$ and \(N=10\), which means the system is 10-dimensional. 

In this 10-dimensional system, we generate 50000 trajectories in total. The training set contains 45000 data, and the test set contains 5000. The distribution of the initial states follows the uniform distribution $\mathcal{U}([-2,2]^{10})$. As for the numerical solver, we take DOP853, the explicit Runge-Kutta method of order 8. In implementation, we use \texttt{scipy} and the time step size $\Delta t=10^{-5}$. The $(u_1, u_2)$ trajectory of this system is plotted in Fig. \ref{fig:lorenz96_vis}.

\begin{figure}
    \centering
    \includegraphics[width=0.5\linewidth]{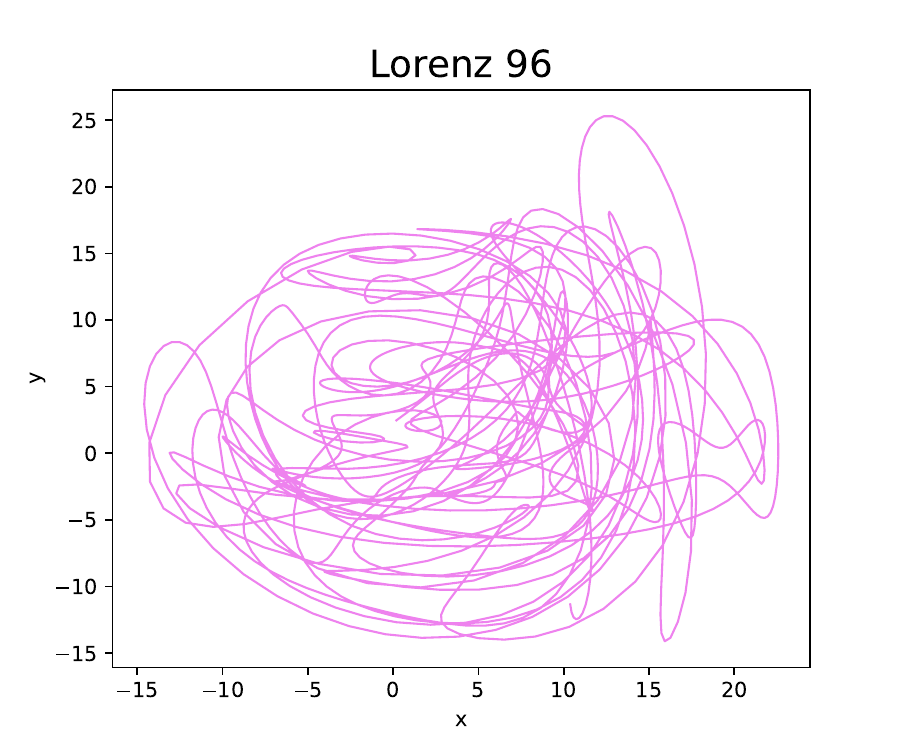}  
    \caption{\textbf{Visualization of one trajectory in the dataset of the Lorenz96 system.}}
    \label{fig:lorenz96_vis}
\end{figure}

\subsection{Planetary System}

In this experiment, we generate 49999 trajectories of the 3-body star–planet system consisting of one central star and three planets, considering the scattering events and collisions. We take 45000 trajectories for training and the remaining 4999 as the test set. All quantities are expressed in astronomically convenient units: stellar mass in solar masses (M$_\odot$), planet mass in solar masses, distances in astronomical units (AU), and angles in radians unless otherwise noted. The initial states of this system are as follows.

\textbf{Central star.}
The stellar mass and radius are fixed as
$
M_\star = 1.0~\mathrm{M}_\odot,\ R_\star = 0.00465~\mathrm{AU}.
$

\textbf{Planet masses and radii.}
Let $\mathrm{M}_{\rm Jup}=10^{-3}\,\mathrm{M}_\odot$ denote one Jupiter mass. 
We first draw the mass of the innermost planet $m_1$ from a log-uniform prior,
$
m_1 \sim \mathrm{LogUniform}\!\big(10^{-3}, 3\big)\times \mathrm{M}_{\rm Jup},
$
and set the next two planet masses within $\pm10\%$ of $m_1$:
$
m_2 = m_1\,u_2, m_3 = m_1\,u_3,\text{where}\ u_2,u_3 \sim \mathcal{U}(0.9,1.1).
$
Planetary radii are assigned using a three-regime broken power law following \citet{Muller24}, derived from the PlanetS catalog, which includes only planets with reliably measured masses and radii:
\begin{align*}
R(m)=
\begin{cases}
1.02\;\Big(\dfrac{m}{0.003\,\mathrm{M}_{\rm Jup}}\Big)^{0.27}, & m < 4.37\times 0.003\,\mathrm{M}_{\rm Jup},\\[8pt]
0.56\;\Big(\dfrac{m}{0.003\,\mathrm{M}_{\rm Jup}}\Big)^{0.67}, & 4.37\times 0.003\,\mathrm{M}_{\rm Jup}\le m < 127\times 0.003\,\mathrm{M}_{\rm Jup},\\[8pt]
18.6\;\Big(\dfrac{m}{0.003\,\mathrm{M}_{\rm Jup}}\Big)^{-0.06}, & m \ge 127\times 0.003\,\mathrm{M}_{\rm Jup},
\end{cases}
\end{align*}
where $R$ is in units of Earth radius. Applying $R(m)$ to $(m_1,m_2,m_3)$ yields $(R_1,R_2,R_3)$.

\textbf{Semi-major axes.}
The innermost semi-major axis is drawn log-uniformly,
$
a_1 \sim \mathrm{LogUniform}(0.02,0.40)~ \text{AU}.
$
The outer planets are initialized by enforcing a fixed spacing of $K$ mutual Hill radii between adjacent orbits. We adopt $K = 4$ to promote efficient dynamical scattering; the resulting orbital distribution remains largely insensitive to the precise spacing, provided that instability and scattering occur \citep{Chatterjee08}. 

Specifically, let $a_k$ and $a_{k+1}$ be adjacent semi-major axes with planet masses $m_k$ and $m_{k+1}$
Their mutual Hill radius is
$
R_{\rm H,mut} = \Big(\frac{m_k+m_{k+1}}{3M_\star}\Big)^{1/3}\frac{a_k+a_{k+1}}{2}.
$
We choose $a_{k+1}>a_k$ such that
$
a_{k+1}-a_k \;=\; K\,R_{\rm H,mut}, K=4,
$
which implicitly defines $a_2$ from $(a_1,m_1,m_2)$ and $a_3$ from $(a_2,m_2,m_3)$.

\textbf{Orbital elements.}
Eccentricities, mean anomalies, arguments of pericenter, inclinations, and longitudes of ascending node are sampled independently as
\begin{align*}
&e_j \sim \mathcal{U}(0.01,0.05), \qquad M_j \sim \mathcal{U}(0,2\pi), \qquad \omega_j \sim \mathcal{U}(0,2\pi),\\
&i_j \sim \mathcal{U}(0^\circ,3^\circ), \qquad \Omega_j \sim \mathcal{U}(0,2\pi), \qquad j=1,2,3.
\end{align*}

As for the numerical solver, we use TRACE \cite{lu2024trace}, a hybrid reversible integrator for planetary dynamics with arbitrary close encounters, in the N-body package \texttt{REBOUND} \cite{rebound}. The collision is performed with \verb|sim.collision = `direct'; sim.collision_resolve = `merge'| in \texttt{REBOUND}. If a planet's distance from the star exceeds 1000 AU, then it will be considered ejected and removed from the simulation. With the \texttt{gr-potential} option in \texttt{REBOUNDx}, we also include the effects of apsidal precession due to general relativity \cite{tamayo2020reboundx}. 

\subsection{Real-world Globular Cluster System}
\textbf{Astrophysical background.}Globular clusters (GCs) are massive, dense, self-gravitating stellar systems that typically contain about $10^4$--$10^6$ stars together with a population of stellar remnants \citep{Harris1996CatalogGC,Heggie2003GravitationalMB}. They are among the oldest stellar systems surviving in the present-day Universe, with typical ages of order $\sim 10~\mathrm{Gyr}$ or older \citep{VandenBerg2013AgesGC}. Because their ages, metallicities, orbits, and dynamical properties retain information about early star formation and accretion events, Galactic GCs are thought to be closely linked to the early assembly history of the Milky Way, including both in-situ formation and accretion through early galaxy mergers \citep{Massari2019OriginGC}.

GCs also provide unique laboratories for studying stellar dynamics in dense environments \citep{Hut1992BinariesInGC}. Their high stellar densities make close encounters, binary interactions, mergers, and physical collisions dynamically important over long timescales. These processes can form or transform exotic stellar populations, including blue straggler stars \citep{Bailyn1995BlueStrag}, low-mass X-ray binaries \citep{Pooley2003CloseBin}, millisecond pulsars \citep{Ransom2008PulsarsGC}, and possibly intermediate-mass black holes \citep{Miller2002ProducIMBH,Mezcua2017ObsIMBH}. These properties make GCs an important astrophysical system for studying the coupled effects of long-term relaxation, stellar evolution, binary dynamics, and external tidal fields.

\textbf{Existing approaches and computational challenges.} A central goal in GC studies is to infer the initial conditions and evolutionary histories that could have produced the present-day cluster population. This inverse problem has been approached for decades through models of GC mass-function evolution, direct and Monte Carlo dynamical simulations, and formation models linking present-day GC systems to high-redshift star formation and galaxy assembly \citep{Vesperini1998GCMF,Baumgardt2003DynamicalEvolution,Chatterjee2013DynamicalState,Kruijssen2015GCFormation,Forbes2018GCFormation,Pijloo2015InitialConditions}. However, a Hubble time of stellar evolution, two-body relaxation, mass segregation, binary interactions, tidal stripping, and possible core collapse can erase or strongly reshape memory of the birth state, making the mapping from initial cluster parameters to present-day observables highly degenerate \citep{VesperiniHeggie1997IMF,Baumgardt2003DynamicalEvolution,Giersz2011MonteCarlo47Tuc,Pijloo2015InitialConditions}.

Forward modelling GC evolution is itself a major challenge in theoretical astrophysics. GCs are collisional gravothermal systems, for which the cumulative effects of distant gravitational encounters must be modelled over many relaxation times \citep{Spitzer1987DynEvolGC}. Direct $N$-body simulations of realistic GCs are computationally demanding, with the computational cost increasing steeply with particle number, approximately as $N^{7/3}$ per dynamical time \citep{Makino1988PerformNbody,Wang2016Dragon}. Monte Carlo methods therefore provide a practical alternative by statistically modelling the cumulative effects of two-body relaxation and reproducing the bulk evolution of direct $N$-body models with greatly reduced computational cost \citep{Henon1971MCModel,Giersz2013MOCCA,Rodriguez2016CMC}. Nevertheless, realistic Monte Carlo models can still require hundreds to thousands of CPU hours for Hubble-time integrations and scale rapidly with particle number and added physical ingredients \citep{Kremer2020CMCCatalog,Prieto2024IMBHProgenitor}. These costs make brute-force inverse inference over large initial-condition spaces prohibitive and motivate a distributional, simulation-trained approach such as Bi-CFM.

\textbf{Simulation details.} We simulated 124 initial conditions and augmented the resulting sample with 33 publicly available trajectories from previous studies. For each trajectory, we retained snapshots at ages of at least $9\,\mathrm{Gyr}$ to construct the dataset.
As for the sampling of physical parameters, the initial virial radius $r_{\rm v}$ is sampled log-uniformly over $0.5$--$4.0\,{\rm pc}$. The Galactocentric radius $R_{\rm g}$ is sampled uniformly over $2.0$--$25.0\,{\rm kpc}$. And the metallicity is sampled uniformly over $-2.0 \leq [{\rm M/H}] \leq 0.0$. 
All other initial conditions follow \citet{Kremer2020CMCCatalog}: clusters are initialized as King models with $W_0=5$, a Kroupa initial mass function, and an initial binary fraction of $f_{\rm b}=5\%$. Binary secondary masses are drawn from a flat mass-ratio distribution over $0.1 \leq q \leq 1$, orbital periods are drawn from a log-flat distribution, and eccentricities follow a thermal distribution. 

\section{Details of Model Implementations}

In this section, we introduce the details of model implementations, including computational resources and implementations of our method and baselines.

\subsection{Computational Resources}
All GPU experiments are conducted on a single NVIDIA A6000 card.
The server is equipped with two Intel third-generation Ice Lake processors, each running at a base frequency of at least 2.6 GHz and a thermal design power (TDP) of no less than 250 W, providing a total of 128 CPU cores.

\subsection{Bi-CFM and CBi-CFM}

\textbf{Details of the conditional CFM and the implementation of bidirectional training.} To begin with, we introduce the CFM for generating the conditional probability $p(\mathbf{x} \mid \mathbf{c})$. We note that, in this conditional generation $p(\mathbf{x}|\mathbf{c})$, $\mathbf{c}$ represents an external condition provided by the user (\textit{e.g.}, a label or context), whereas in the original CFM as stated in Section \textcolor{blue}{4.1}, the condition $z=(\mathbf{x}^0,\mathbf{x}^1)$ is the endpoints of the flow path used to compute the conditional vector field instead of the marginal vector field. Following the formulation of CFM introduced in Section \textcolor{blue}{4.1}, we can extend all intermediate quantities, namely \( p_\tau(\mathbf{x}) \), \( p_\tau(\mathbf{x}\mid\mathbf{z}) \), \( v^\tau(\mathbf{x}) \), \( v^\tau(\mathbf{x}\mid\mathbf{z}) \), and \( v_\theta(\tau,\mathbf{x}) \), to be conditioned on \( \mathbf{c} \). As a result, these quantities become \( p_\tau(\mathbf{x}\mid\mathbf{c}) \), \( p_\tau(\mathbf{x}\mid\mathbf{z}, \mathbf{c}) \), \( v_\tau(\mathbf{x}\mid\mathbf{c}) \), \( v_\tau(\mathbf{x}\mid\mathbf{z},\mathbf{c}) \), and \( v_\theta(\tau, \mathbf{x},\mathbf{c}) \), while the derivation still holds. This conditioning allows the model to approximate the target distribution \( p(\mathbf{x} \mid \mathbf{c}) \), so the corresponding loss function is 
\begin{equation}
\mathcal{L}_{p(\mathbf{x}\mid\mathbf{c})} = \mathbb{E}_{\tau\sim p(\tau),(\mathbf{x}^\tau,\mathbf{z},\mathbf{c})\sim p_\tau(\mathbf{x}^\tau|\mathbf{z},\mathbf{c})p(\mathbf{z})p(\mathbf{c})}\left[\| v_\theta(\tau,\mathbf{x}^\tau, \mathbf{c}) - v_\tau(\mathbf{x}^\tau|\mathbf{z}, \mathbf{c}) \|_2^2\right].
\end{equation}

Next, we describe how to jointly learn \( p(\mathbf{u}_0^\text{gt} \mid \mathbf{u}_T^\text{gt}) \) and \( p(\mathbf{u}_T^\text{gt} \mid \mathbf{u}_0^\text{gt}) \) in detail. 
For a batch of paired data \(\{(\mathbf{u}_{k,0}^\text{gt}, \mathbf{u}_{k,T}^\text{gt})\}_{k=1}^{K}\), where $k$ is the index of data, we randomly select \(k_0\) samples to train \(p(\mathbf{u}_T^\text{gt} \mid \mathbf{u}_0^\text{gt})\)  and use the remaining \(K - k_0\) samples to train \(p(\mathbf{u}_0^\text{gt} \mid \mathbf{u}_T^\text{gt})\). 
Specifically, we treat each pair \((\mathbf{u}_0^\text{gt}, \mathbf{u}_T^\text{gt})\) as a single sample, so that both \(\mathbf{x}\) and \(v^\tau\) belong to \(\mathbb{R}^{2n}\) in this formulation. For each sample, we assign the corresponding time pair \((\tau_1, \tau_2)\) to it, where \((\tau_1, \tau_2)\) represents the temporal variables associated with each sample.
For the first \(k_0\) samples, we interpolate \(\mathbf{u}_T^\text{gt}\) along the probability path $\mathbf{\tilde{u}}_T = \alpha(\tau_2)\mathbf{x}^0 + \beta(\tau_2)\mathbf{u}_T^\text{gt} + \eta(\tau_2)\boldsymbol{\epsilon}$, where \(\tau_2\) is sampled from \(p(\tau)\) and \(\tau_1 = 1\), indicating that \(\mathbf{u}_0^\text{gt}\) serves as a clean conditioning variable. 
For the remaining \(K - k_0\) samples, we apply a symmetric procedure to learn the reverse conditional distribution \(p(\mathbf{u}_0^\text{gt} \mid \mathbf{u}_T^\text{gt})\). 
In this case, we interpolate \(\mathbf{u}_0^\text{gt}\) along the path
$\alpha(\tau_1)\mathbf{x}^0 + \beta(\tau_1)\mathbf{u}_0^\text{gt} + \eta(\tau_1)\boldsymbol{\epsilon}$
, where \(\tau_1\) is sampled from \(p(t)\) and \(\tau_2 = 1\), 
representing that \(\mathbf{u}_T\) now acts as the clean conditioning variable. 

\textbf{Hyperparameters.} We normalize the data by subtracting the mean value and dividing by the standard deviation of different features. As for the model architecture, we adopt the Multilayer Perception (MLP) as the base model. The time of CFM is repeated and concatenated to other inputs as a channel. We set the MLP to have 6 layers with 1024 nodes. The activation function is the Scaled Exponential Linear Unit (SELU), which induces self-normalizing properties \cite{klambauer2017self}. The hyperparameters are presented in Table \ref{tab:para}. During training, we take $\eta(t)=10^{-5}$, batch size $=1024$, learning rate $10^{-4}$, and 600000 updates. As for the prior distribution, on the planetary system, we take 6 random walks with step size $=0.002$ and 5 projection steps. During sampling, we set the number of sampling steps to 100 and use the Dormand–Prince (RKDP) method \cite{DORMAND198019} in \texttt{Torchdyn} \cite{poli2021torchdyn}. In addition, in the planetary system experiment, in order to make the optimization stable, we clip the gradient of optimization to 1. Additionally, we post-process the output of sampling by clipping its upper bound to the maximum value of the training dataset, allowing us to filter out outliers.

\begin{table}[]
\caption{Hyperparameters of Bi-CFM and CBi-CFM.}\label{tab:para}
\begin{tabular}{@{}llrr@{}}
\toprule
Phase & Hyperparameter & Lorenz, Circuit, Lorenz96 & Planetary \\ \midrule
\multirow{5}{*}{Training} 
& $\eta(t)$ & $10^{-5}$ & $10^{-5}$\\ 
& Batch Size & 1024 & 1024\\
& Learning Rate & $10^{-4}$ & $10^{-4}$\\ 
& Number of Updates & 600000 & 600000\\ 
& Gradient Clip & - & 1\\ 
\midrule
\multirow{3}{*}{\shortstack{Prior Distribution\\of CBi-CFM}} 
& Number of Random Walks & - & 6\\ 
& Step Size $\Delta r$ & - & 0.002\\ 
& Number of Projection Steps & - & 5\\ 
\midrule
\multirow{2}{*}{Sampling} 
& Number of Sampling Steps & 100 & 100\\ 
& Solver & Dormand–Prince & Dormand–Prince\\ 
\botrule
\end{tabular}
\end{table}

\subsection{Baselines}
Below, we report the details of the baseline implementation.

\textbf{Backward Integration.} We take the same traditional solver as the one used for data generation in each system. And we set the time step size to be $10^{-5}$, which corresponds to the best performance. In addition, in order to accelerate this inference time, we use the \texttt{multiprocessing} package to fully leverage multiple processors.

\textbf{Backbone.} We normalize the data in the same way as Bi-CFM. And we utilize the same model architecture as the base model of Bi-CFM and CBi-CFM for fair comparison. The number of nodes and layers is also the same. During training, the batch size is 2048, the learning rate is $10^{-4}$, and the number of updates is $10000$.

\textbf{Random.} We perform Random by randomly changing the order of states at different time steps, so that the correspondence is changed, while the distributions at each time step remains the same as the ground truth.

\begin{figure}[t]
    \centering
    \begin{subfigure}[t]{1\textwidth}
        \centering
        \includegraphics[width=0.6\linewidth]{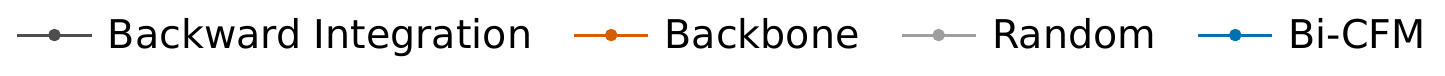}
    \end{subfigure}
    \begin{subfigure}[t]{0.35\textwidth}
        \centering
        \includegraphics[width=\linewidth]{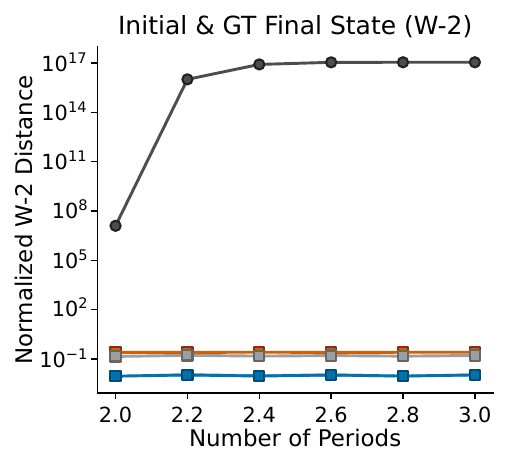}
        \vspace{-15pt}
    \end{subfigure}
    \hspace{0.08\linewidth}
    \begin{subfigure}[t]{0.35\textwidth}
        \centering
        \includegraphics[width=\linewidth]{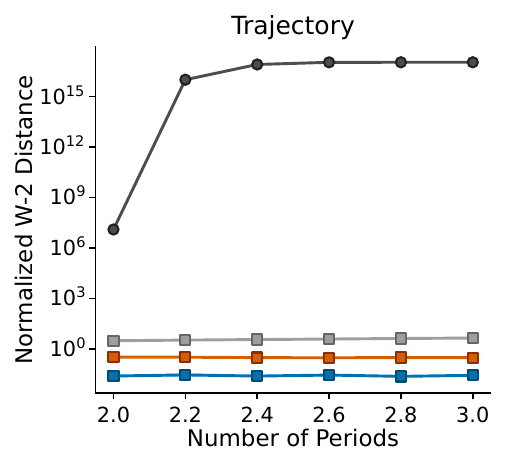}
        \vspace{-15pt}
    \end{subfigure}
    \caption{Evaluation metrics versus horizon length on the Lorenz system.}
    \label{fig:lorenz_metric}
\end{figure}

\begin{figure}[t]
    \centering
    \begin{subfigure}[t]{1\textwidth}
        \centering
        \includegraphics[width=0.6\linewidth]{figs/method_color_legend.pdf}
    \end{subfigure}
    \begin{subfigure}[t]{0.35\textwidth}
        \centering
        \includegraphics[width=\linewidth]{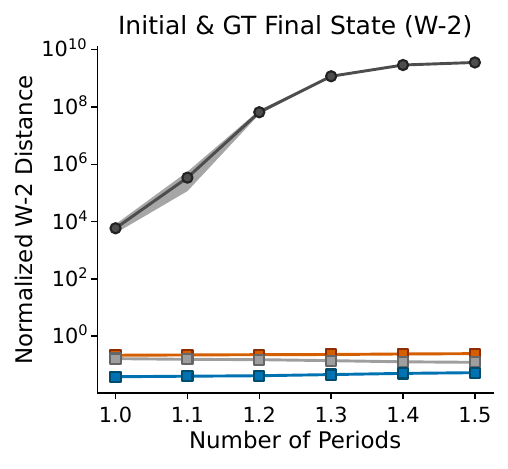}
        \vspace{-15pt}
    \end{subfigure}
    \hspace{0.08\linewidth}
    \begin{subfigure}[t]{0.35\textwidth}
        \centering
        \includegraphics[width=\linewidth]{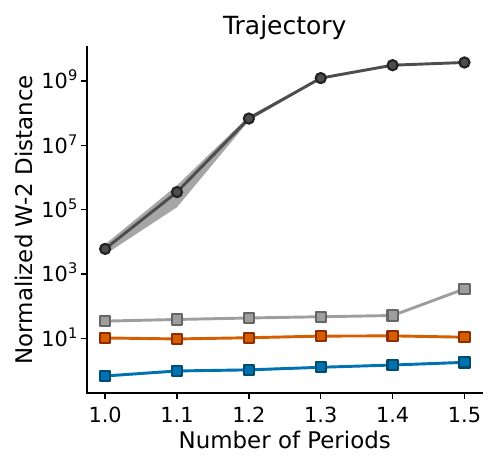}
        \vspace{-15pt}
    \end{subfigure}
    \caption{Evaluation metrics versus horizon length on the circuit system.}
    \label{fig:circuit_metric}
\end{figure}

\begin{figure}[t]
    \centering
    \begin{subfigure}[t]{1\textwidth}
        \centering
        \includegraphics[width=0.6\linewidth]{figs/method_color_legend.pdf}
    \end{subfigure}
    \begin{subfigure}[t]{0.35\textwidth}
        \centering
        \includegraphics[width=\linewidth]{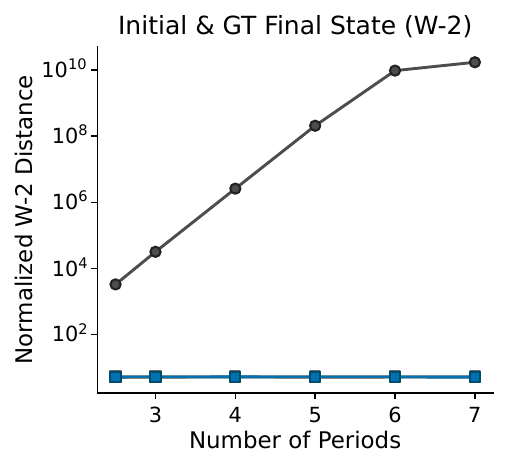}
        \vspace{-15pt}
    \end{subfigure}
    \hspace{0.08\linewidth}
    \begin{subfigure}[t]{0.35\textwidth}
        \centering
        \includegraphics[width=\linewidth]{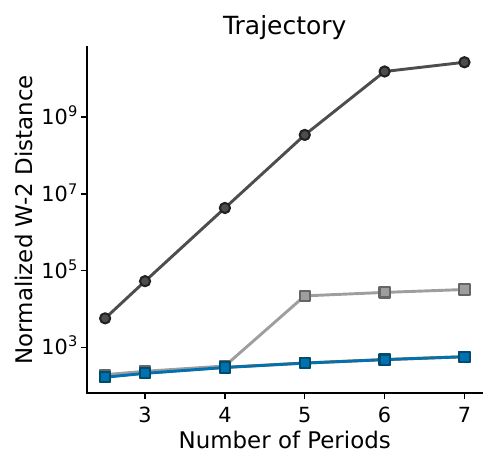}
        \vspace{-15pt}
    \end{subfigure}
    \caption{Evaluation metrics versus horizon length on the Lorenz 96 system.}
    \label{fig:lorenz96_metric}
\end{figure}

\begin{figure}[t]
    \centering
    \begin{subfigure}[t]{1\textwidth}
        \centering
        \includegraphics[width=\linewidth]{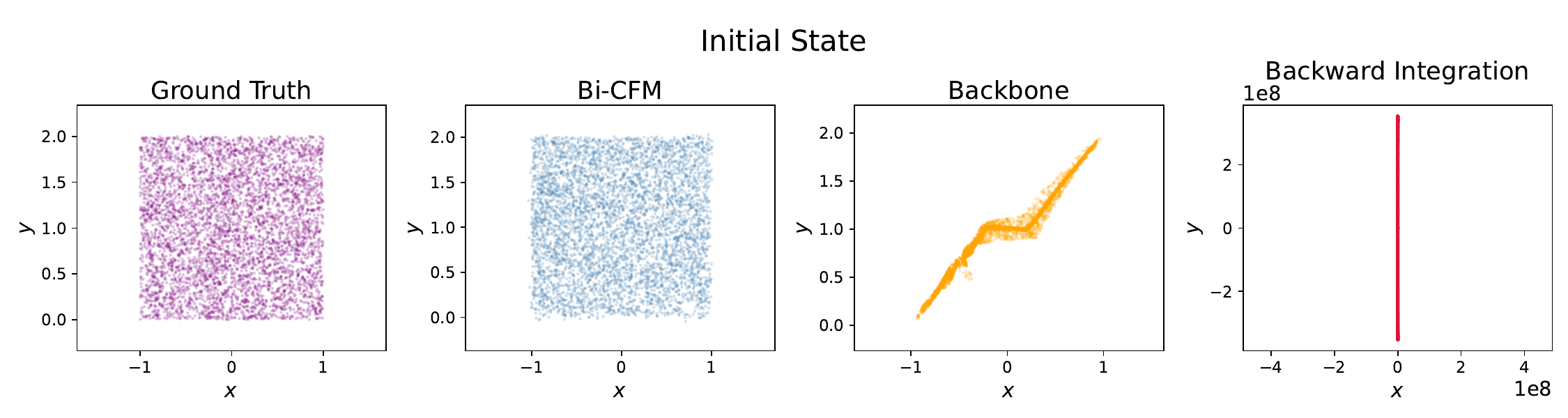}
        \vspace{-15pt}
    \end{subfigure}
    \hfill
    \begin{subfigure}[t]{1\textwidth}
        \centering
        \includegraphics[width=\linewidth]{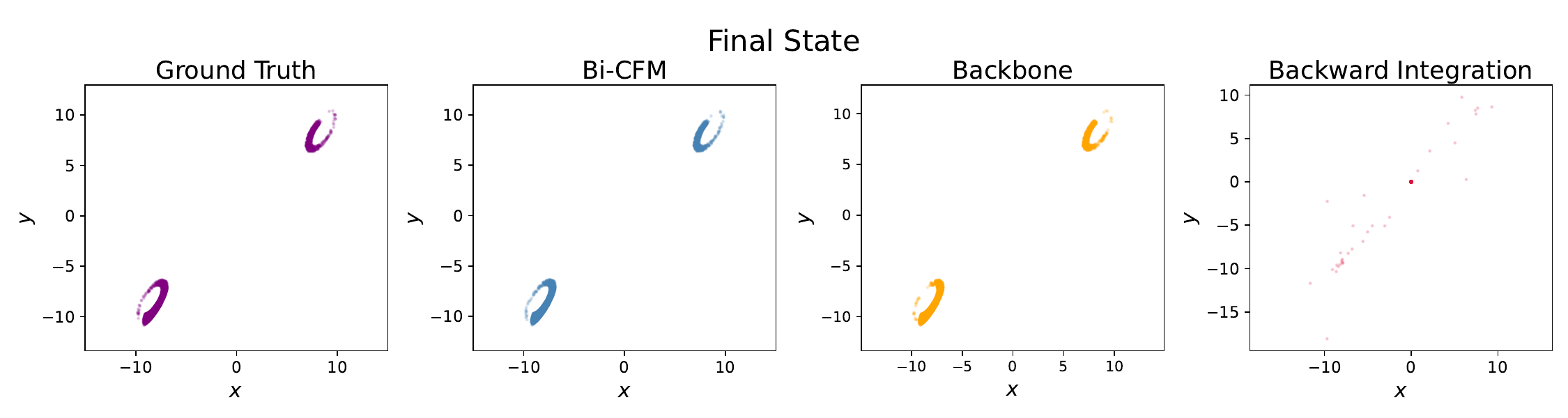}
        \vspace{-15pt}
    \end{subfigure}
    \caption{Visualization of $(x,y)$-value distributions of the Lorenz system.}
    \vspace{-15pt}
    \label{fig:lorenz_xy}
\end{figure}

\begin{figure}[t]
    \centering
    \begin{subfigure}[t]{1\textwidth}
        \centering
        \includegraphics[width=\linewidth]{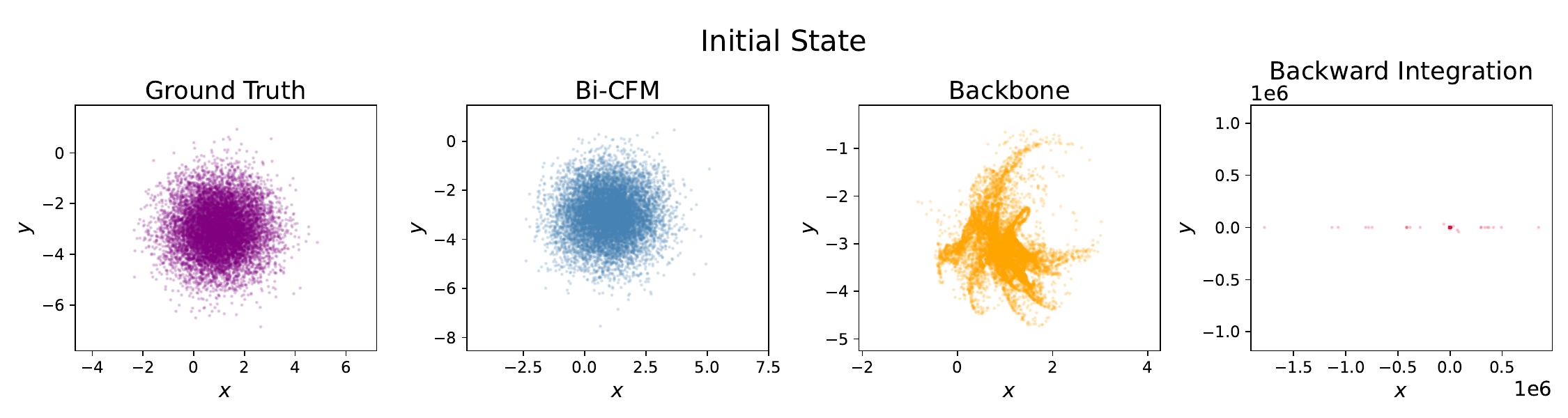}
        \vspace{-15pt}
    \end{subfigure}
    \begin{subfigure}[t]{1\textwidth}
        \centering
        \includegraphics[width=\linewidth]{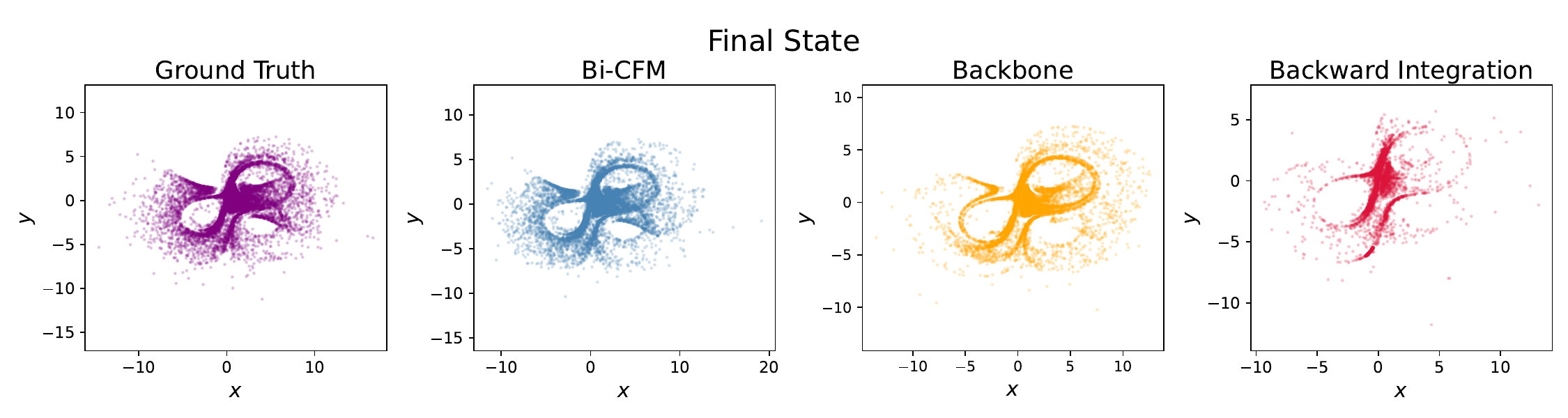}
        \vspace{-15pt}
    \end{subfigure}
    \caption{Visualization of $(x,y)$-value distributions of the Circuit system.}
    \label{fig:circuit_xy}
\end{figure}

\begin{figure}
    \vspace{-15pt}
    \centering
    \includegraphics[width=\linewidth]{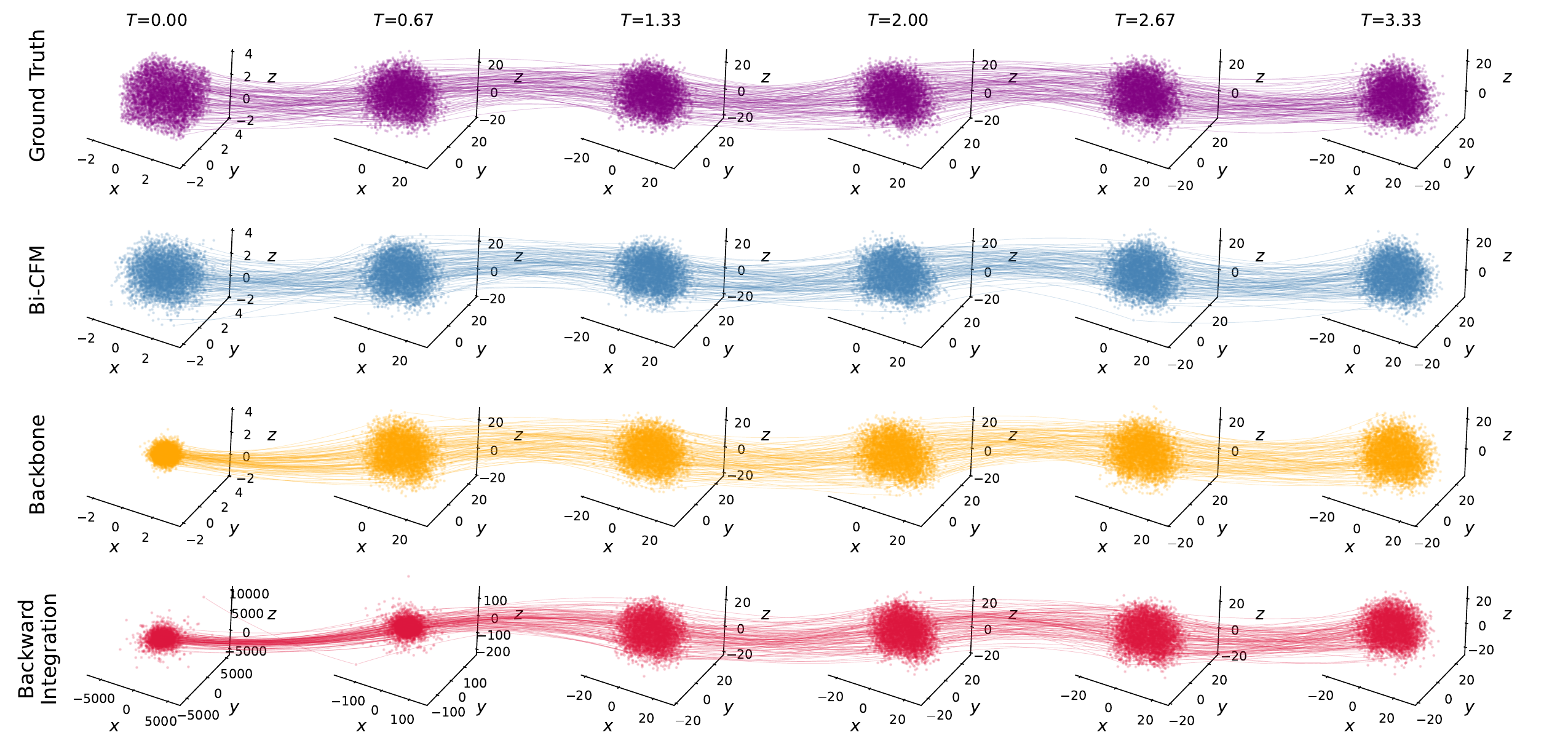}  
    \caption{\textbf{Trajectories of $(u_0,u_1,u_2)$-value of the Lorenz 96 system.}}
    \label{fig:lorenz96_traj}
\end{figure}

\begin{figure}
    \vspace{-15pt}
    \centering
    \begin{subfigure}[t]{1\textwidth}
        \centering
        \includegraphics[width=\linewidth]{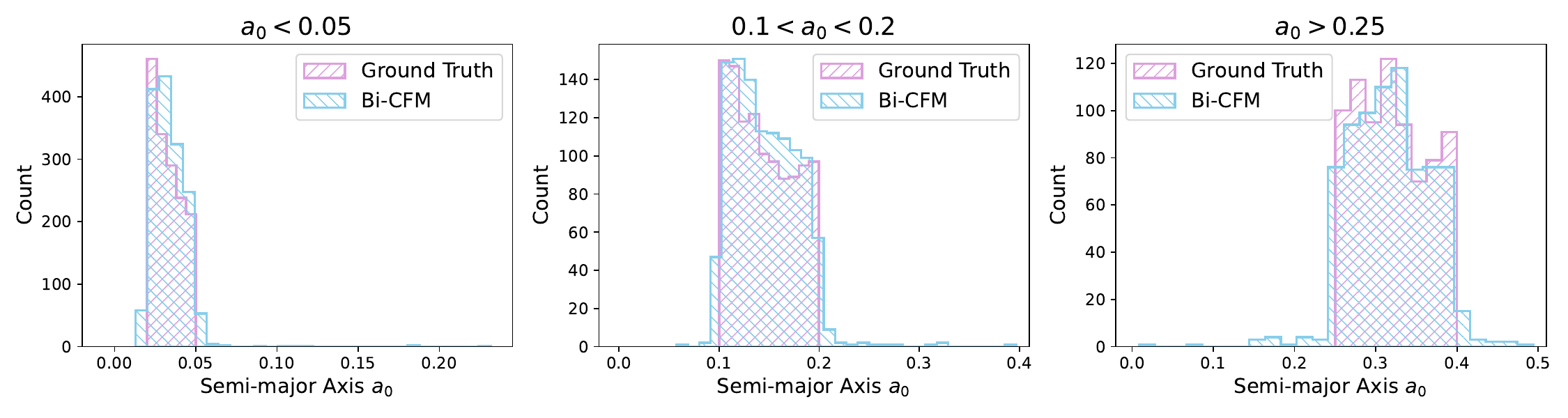}
        \vspace{-15pt}
    \end{subfigure}
    \hfill
    \begin{subfigure}[t]{1\textwidth}
        \centering
        \includegraphics[width=\linewidth]{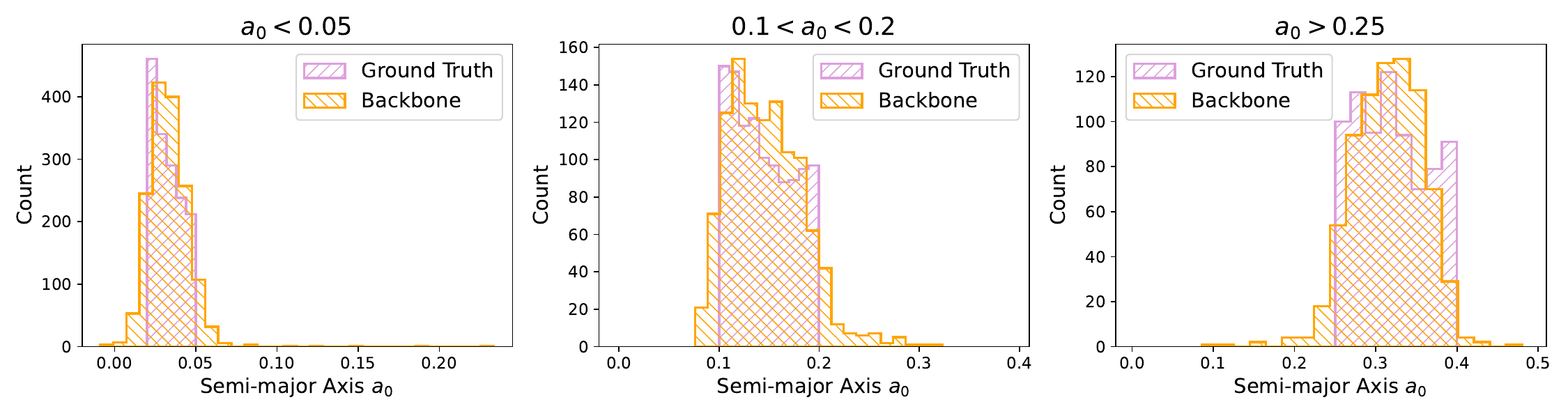}
        \vspace{-15pt}
    \end{subfigure}
    \caption{\textbf{Distributions of inferred initial semi-major axis by Bi-CFM and Backbone, according to trajectories with small, middle, and large ground truth initial semi-major axis.}}
    \label{fig:a0}
\end{figure}

\begin{figure}
    \vspace{-15pt}
    \centering
    \includegraphics[width=0.8\linewidth]{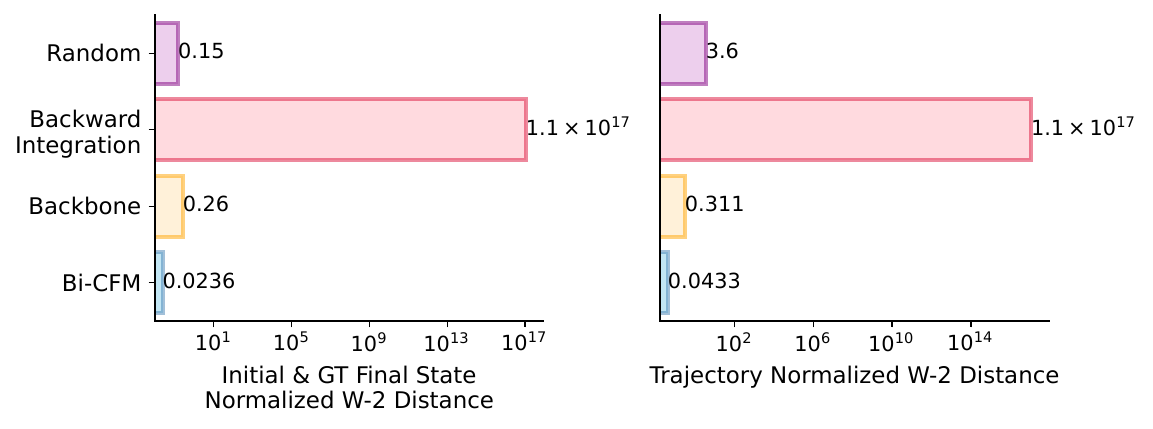}  
    \caption{\textbf{W-2 distance of initial-final state pairs and trajectories on the Lorenz system with noise.}}
    \label{fig:robust}
\end{figure}

\section{Other Visualizations}

\textbf{Evaluation metrics of three chaotic systems.} Owing to space limitations in the main text, in Fig. \ref{fig:lorenz_metric}, \ref{fig:circuit_metric}, and \ref{fig:lorenz96_metric}, we report the other evaluation metrics for the three chaotic systems, including the W-2 distance between paired initial and target final states and the Wasserstein-2 distance between trajectories. As illustrated in the figures, across different systems and evaluation metrics, Bi-CFM consistently achieves the lowest error. Moreover, its error increases only gradually over time, whereas that of Backward Integration grows exponentially.

\newpage
\textbf{Visualizations of the $(x,y)$-value of the inferred initial and evolved final states.} To provide complementary perspectives on the distributions of initial and final states obtained by different methods, we present the distributions in the $(x,y)$ plane for the Lorenz and Circuit systems in Fig. \ref{fig:lorenz_xy} and \ref{fig:circuit_xy}, respectively. From these two figures, we both observe that Bi-CFM closely matches the distribution of the ground truth on both systems, while the other two mismatch the ground truth. Especially, the range of Backward Integration's value is larger than the ground truth. And the initial distributions of Backbone on both systems are close to a combination of line segments.

\textbf{Trajectories of $(u_0,u_1,u_2)$-value of the Lorenz 96 system.} Because of space limitations in the main text, we provide the trajectory visualizations for the Lorenz 96 system in Fig. \ref{fig:lorenz96_traj}. Since Lorenz 96 is high-dimensional, we visualize only its first three dimensions. At $T=0$, only Bi-CFM produces a state distribution consistent with the ground truth, whereas at later times the results of different methods become visually similar. This may be because the high dimensionality and chaoticity of the system make it difficult to distinguish different methods using such low-dimensional visualizations alone.

\textbf{Results of inferring the corresponding interval with Backbone and Bi-CFM on the planetary system.} In addition to the CBi-CFM results presented in the main text, we also examine whether Bi-CFM and Backbone can correctly infer which interval of the initial semi-major axis $a_0$ corresponds to a given final state. From visualizations of the inferred $a_0$ in Fig. \ref{fig:a0}, we can observe that both Bi-CFM and Backbone can distinguish the interval of $a_0^\text{gt}$ according to the final states. However, as shown, the Backbone infers more $a_0$ values outside the corresponding interval, leading to a broader distribution compared with the ground truth. In contrast, the distributions obtained by Bi-CFM fit better with the ground truth, indicating that Bi-CFM learns the causal relationship more effectively than the base model Backbone.

\textbf{Results of robustness analysis.} After adding Gaussian noise with a scale of $0.1$ to the target final states, we re-evaluate all metrics. In the main text, we report the results for three metrics, while the remaining two are provided in Fig. \ref{fig:robust}. Bi-CFM still maintains the lowest error level, whereas Backward Integration performs the worst.

\section{Proof of Proposition 1}

Below, we provide the proof of Proposition 1 in the Method section (Section \textcolor{blue}{4}).

\begin{proposition}
Let $H:\mathbb{R}^n\to\mathbb{R}$ be continuously differentiable, and let the conservation manifold be defined as
\[
\mathcal{M} = \{\mathbf{u}\in\mathbb{R}^n : H(\mathbf{u}) = H(\mathbf{u}_T^\text{gt})\},
\]
with $\nabla H(\mathbf{u}) \neq \mathbf{0}$. With $P$ defined as 
\[
P(\mathbf{u}) \,=\, I \;-\; \frac{\nabla H(\mathbf{u})\,\nabla H(\mathbf{u})^{\top}}{\|\nabla H(\mathbf{u})\|^2},
\]
for any $\mathbf{r}\in\mathbb{R}^n$, the vector $\mathbf{r}\cdot P(\mathbf{u})$ lies in the tangent space $T_{\mathbf{u}}\mathcal{M}$ of the manifold at $\mathbf{u}$.
\end{proposition}

\begin{proof}
For the manifold $\mathcal{M}$, the tangent space at $\mathbf{u}\in\mathcal{M}$ is
\[
T_{\mathbf{u}}\mathcal{M} = \{\mathbf{v}\in\mathbb{R}^n : \mathbf{v}\nabla H(\mathbf{u}) = 0\}.
\]
Let $\mathbf{v} = \mathbf{r}P(\mathbf{u})$ for an arbitrary $\mathbf{r}\in\mathbb{R}^n$. Using the definition of $P(\mathbf{u})$, we have
\begin{align*}
\mathbf{v}\nabla H(\mathbf{u}) &= \mathbf{r}P(\mathbf{u})\nabla H(\mathbf{u}) \\
&= \mathbf{r}\Big(\nabla H(\mathbf{u}) - \frac{\nabla H(\mathbf{u})\nabla H(\mathbf{u})^{\top}\nabla H(\mathbf{u})}{\|\nabla H(\mathbf{u})\|^2}\Big) \\
&= \big(\nabla H(\mathbf{u}) - \nabla H(\mathbf{u})\big)\mathbf{r} \\
&= \mathbf{0}.
\end{align*}
Hence $\mathbf{v}\in T_{\mathbf{u}}\mathcal{M}$, which proves that $P(\mathbf{u})\,\mathbf{r}$ lies in the tangent plane of the conservation manifold at $\mathbf{u}$. 
\end{proof}

\section{Analyze Chaoticity with the Lyapunov Spectrum}

In this section, we introduce the Lyapunov spectrum analysis, which is a tool for analyzing the chaoticity. With it, we further discuss the chaotic behaviors of forward and reverse dynamics.

\subsection{Asymmetric Chaoticity of Forward and Reverse Dynamics}
Herein, we analyze the asymmetric chaotic behaviors of forward and reverse dynamics with Lyapunov Spectrum Analysis. For a dynamical system with the governing equation
\begin{align*}
    \dot{\mathbf{u}} = f(\mathbf{u}),\quad \mathbf{u}\in\mathbb{R}^n,
\end{align*}
the Lyapunov spectrum $\{ \lambda_1, \lambda_2, \ldots, \lambda_n \}\ (\lambda_1\geq\lambda_2\geq\cdots\geq\lambda_n)$ provides a measure of the system’s stability and degree of chaos by quantifying the exponential rates at which nearby trajectories diverge or converge in different directions of the tangent space. Below, we discuss the Lyapunov spectrum of the forward and reverse dynamics theoretically to demonstrate that the degrees of the two dynamics' chaos are asymmetric and the reverse dynamics can be more chaotic than the forward dynamics. The Lyapunov spectrum is defined from the solution $\Phi(t)$ of the variational equation
\begin{align}
    \dot{\Phi}(t) = J(\mathbf{u}(t)) \Phi(t), \quad \Phi(0) = I,
    \label{eq:phi}
\end{align}
where $J(x) = Df(x)$ is the Jacobian matrix of the system and $I$ is the identity matrix. $\Phi(t)$ characterizes the evolution of perturbations through 
\begin{align*}
\delta \dot{\mathbf{u}}(t) = \Phi(t)\, \delta \mathbf{u}(t),
\end{align*}
and the magnitude of $\delta \mathbf{u}(t)$ consequently grows as $\exp(\Phi(t))$.
Then, the Lyapunov matrix $\Lambda$ is defined as
\begin{align*}
    \Lambda = \lim_{t \to \infty} \frac{1}{2t} \log\left[ \Phi(t)\Phi^T(t) \right],
\end{align*}
and the Lyapunov exponents $\lambda_i\ (i=1,\cdots,n)$ are given by the eigenvalues of $\Lambda$ \cite{lyapunov1892general, oseledets1968multiplicative}.

The maximal Lyapunov exponent $\lambda_1$ quantifies the exponential rate of divergence along the most unstable direction, thus it characterizes the dominant mode of sensitivity to initial conditions, which is the hallmark of chaos. A system is therefore considered chaotic if its largest Lyapunov exponent is positive, indicating that infinitesimal perturbations grow exponentially in at least one direction, leading to unpredictable long-term behavior.

Now consider the time-reversed dynamics governed by
\begin{align*}
\dot{\mathbf{u}} = -f(\mathbf{u}), \quad \mathbf{u} \in \mathbb{R}^n.
\end{align*}
Then the variational equation is given by
\begin{align}
\dot{\Psi}(t) = -J(\mathbf{u}(t)) \Psi(t), \quad \Psi(0) = I.
\label{eq:psi}
\end{align}
Denote the new Lyapunov matrix as
\begin{align*}
\Sigma = \lim_{t \to \infty} \frac{1}{2t} \log\left[ \Psi(t)\Psi^T(t) \right],
\end{align*}
and the corresponding Lyapunov exponents as $\sigma_1\geq\sigma_2\geq\cdots\geq\sigma_n$.

We can then compare $\Phi(t)$ and $\Psi(t)$. Specifically, according to the Dyson series, the solutions of Eq.~\ref{eq:phi} and Eq.~\ref{eq:psi} are given by
\begin{align*}
\Phi(t) = I &+ \int^t_0 J(t_1)dt_1 + \int^t_0 dt_1 \int^{t_1}_0 dt_2 J(t_1)J(t_2) \\
&+ \int^t_0 dt_1 \int^{t_1}_0 dt_2 \int^{t_2}_0 dt_3 J(t_1)J(t_2)J(t_3) + \cdots, \\
\Psi(t) = I &- \int^t_0 J(t_1)dt_1 + \int^t_0 dt_1 \int^{t_1}_0 dt_2 J(t_1)J(t_2) \\
&- \int^t_0 dt_1 \int^{t_1}_0 dt_2 \int^{t_2}_0 dt_3 J(t_1)J(t_2)J(t_3) + \cdots,
\end{align*}
from which we can see that $\Phi(t)$ and $\Psi(t)$ generally do not have an explicit analytical relationship. This observation further implies that the chaoticity of the forward and reverse evolutions is not identical but inherently unbalanced, leading to our proposed bidirectional modeling strategy. 

Next, we consider a special case to illustrate the relationship between the maximal Lyapunov exponent of the reverse process and that of the forward process. In particular, since
\begin{align*}
\Phi(t)^{-1} = I &- \int^t_0 J(t_1)dt_1 + \int^t_0 dt_1 \int^{t_1}_0 dt_2 J(t_2)J(t_1) \\
&- \int^t_0 dt_1 \int^{t_1}_0 dt_2 \int^{t_2}_0 dt_3 J(t_3)J(t_2)J(t_1) + \cdots
\end{align*}
when $J(t)$ is commutative, that is, when $J(t_1)J(t_2) = J(t_2)J(t_1)$ for any $t_1$ and $t_2$, we have $\Psi(t) = \Phi(t)^{-1}$. By applying singular value decomposition, we obtain the following proposition.
\begin{proposition}
When $J(t)$ is commutative, $\Psi(t) = \Phi(t)^{-1}$. Thus, $\Sigma = -\Lambda$, and $\sigma_1 = \sigma_{\max} = -\lambda_{\min} = -\lambda_n$.
\end{proposition}

\begin{proof}
Let $\Phi(t)$ admit the singular value decomposition
\begin{align*}
\Phi(t) = U(t)\Sigma(t)V(t)^{\mathsf T},
\end{align*}
where the diagonal entries of $\Sigma(t)$ are the singular values $s_i(t)$.
Then we have
\begin{align*}
\Phi(t)\Phi(t)^{\mathsf T} = U(t)\,\Sigma(t)^2\,U(t)^{\mathsf T},
\qquad
\Psi(t) = \Phi(t)^{-1} = V(t)\,\Sigma(t)^{-1}\,U(t)^{\mathsf T},
\end{align*}
and consequently,
\begin{align*}
\Psi(t)\Psi(t)^{\mathsf T} = V(t)\,\Sigma(t)^{-2}\,V(t)^{\mathsf T}.
\end{align*}
Therefore, the Lyapunov matrices
\begin{align*}
\Lambda = \lim_{t\to\infty}\frac{1}{2t}\log\!\big(\Phi(t)\Phi(t)^{\mathsf T}\big),
\qquad
\Sigma = \lim_{t\to\infty}\frac{1}{2t}\log\!\big(\Psi(t)\Psi(t)^{\mathsf T}\big),
\end{align*}
have eigenvalues given respectively by
\begin{align*}
\lambda_i = \lim_{t\to\infty}\frac{1}{t}\log s_i(t),
\qquad
\sigma_i = -\lim_{t\to\infty}\frac{1}{t}\log s_i(t).
\end{align*}
It follows that $\Sigma = -\Lambda$, and thus $\sigma_i = -\lambda_i$.
In particular, the maximal exponent satisfies
\begin{align*}
\sigma_{\max} = -\lambda_{\min}.
\end{align*}
\end{proof}

In this case, if the forward dynamic has at least one contracting direction ($\lambda_n<0$), the corresponding reverse dynamic will exhibit an expansion, leading to $\sigma_0>0$.

\subsection{Numerical Algorithm of Computing the Lyapunov Spectrum}

Based on the above derivations, the Lyapunov spectrum can be computed numerically with discretization. The numerical computational workflow is outlined as follows. Given an arbitrary dynamical system described by $\dot{\mathbf{u}} = f(\mathbf{u})$, we augment the original state with the variational matrix $\Phi(t)\in\mathbb{R}^{d\times d}$ 
that satisfies $\dot{\Phi} = J(\mathbf{u}(t))\Phi(t)$, where $J(\mathbf{u}(t)) = \partial f/\partial \mathbf{u}$ is the Jacobian of the system. The algorithm proceeds through the following stages:

\begin{enumerate}
    \item \textbf{Augmented system construction.}  
    For any given right-hand-side function $f(\mathbf{u})$,  
    the code constructs the extended system by concatenating $\mathbf{u}$ and the vectorized $\Phi(t)$ into a single state vector.  
    The Jacobian $J(\mathbf{u}(t))$ is approximated via finite differences.

    \item \textbf{Initialization.}  
    The initial variational matrix is set to $\Phi(0)=I$.  
    The extended initial condition $X_0=[\mathbf{u}_0,\Phi(0)]$  
    is integrated over the time interval $[0,t_{\max}]$ with discrete reorthogonalization steps of size $\Delta t$.

    \item \textbf{Segment-wise integration.}  
    Within each interval $[t_k, t_{k+1}]$, the system is integrated using the adaptive Runge–Kutta method (RK45) \cite{dormand1980family} with tolerances  
    $\texttt{rtol}=10^{-6}$ and $\texttt{atol}=10^{-9}$.  
    At the end of each segment, the current $\Phi$ matrix is extracted.

    \item \textbf{QR reorthogonalization.}  
    The matrix $\Phi$ is decomposed as $\Phi = QR$ with QR-decomposition in \texttt{numpy.linalg} \cite{harris2020array}, 
    and the logarithms of the absolute diagonal entries of $R$ are accumulated as local stretching rates.  
    The next segment then starts from the reorthogonalized basis $\Phi \leftarrow Q$  
    to maintain numerical stability.

    \item \textbf{Exponent estimation.}  
    After all time segments, the Lyapunov exponents are obtained as  
    \[
        \lambda_i = \frac{1}{t_{\max}}\sum_{k} \log |R_{ii}^{(k)}|,\quad i=1,\dots,d.
    \]  
    The resulting exponents are sorted in descending order to yield the Lyapunov spectrum.
\end{enumerate}

\subsection{Examples}

Below, we present two representative examples that correspond to cases where the reverse dynamics exhibit stronger chaos than the forward process, and where the forward and reverse dynamics possess equivalent levels of chaoticity.

\textbf{Example 1.} We apply the proposed numerical procedure to the classical chaotic Lorenz system \cite{lorenz2017deterministic} to compute both the forward and backward Lyapunov spectra. The results show that in the Lorenz system with $\sigma = 10$, $\rho = 28$, and $\beta = 8/3$ starting from $\mathbf{u}(0)=(0,1,0)$, the forward Lyapunov spectrum is $(\lambda_1, \lambda_2, \lambda_3) = (0.607, 0.068, -14.341)$. In contrast, the reverse Lyapunov spectrum is $(\sigma_1, \sigma_2, \sigma_3) = (13.898, 0.947, -1.180)$, exhibiting a larger maximum exponent than the forward one. Also, we find that the maximal Lyapunov exponent of the inverse process is approximately equal to the negative of the minimal one of the forward process. This finding reveals that the backward dynamics display even stronger chaotic behavior, indicating an intrinsic asymmetry in the degree of chaoticity between forward and inverse evolutions.

\textbf{Example 2.} Next, we consider another classical chaotic system, the three-body problem. Denote the positions of the three mutually gravitating bodies as $\mathbf{r}_i = (x_i, y_i, z_i)$, then the governing equation of forward dynamics is
\begin{align*}
\begin{cases}
\ddot{\mathbf{r}}_1 &= -G m_2 \frac{\mathbf{r}_1 - \mathbf{r}_2}{\|\mathbf{r}_1 - \mathbf{r}_2\|^3}
                      -G m_3 \frac{\mathbf{r}_1 - \mathbf{r}_3}{\|\mathbf{r}_1 - \mathbf{r}_3\|^3}, \\
\ddot{\mathbf{r}}_2 &= -G m_3 \frac{\mathbf{r}_2 - \mathbf{r}_3}{\|\mathbf{r}_2 - \mathbf{r}_3\|^3}
                      -G m_1 \frac{\mathbf{r}_2 - \mathbf{r}_1}{\|\mathbf{r}_2 - \mathbf{r}_1\|^3}, \\
\ddot{\mathbf{r}}_3 &= -G m_1 \frac{\mathbf{r}_3 - \mathbf{r}_1}{\|\mathbf{r}_3 - \mathbf{r}_1\|^3}
                      -G m_2 \frac{\mathbf{r}_3 - \mathbf{r}_2}{\|\mathbf{r}_3 - \mathbf{r}_2\|^3},
\end{cases}
\end{align*}
where $m_i$ is the mass and $G$ is the gravitational constant \cite{marchal1984three}. Considering its backward dynamics, we find that since each $\mathbf{r}_i$ involves the second-order time derivative, the reverse dynamics satisfy exactly the same evolution equations as the forward ones. Consequently, the Lyapunov exponents of the forward and backward dynamics are identical. This implies that the backward dynamics of the three-body system are also chaotic, exhibiting the same degree of chaoticity as the forward evolution.

\end{appendices}

\bibliography{sn-bibliography}%

\end{document}